\documentclass{article}
\usepackage{graphicx}
\usepackage{placeins}
\usepackage[utf8]{inputenc}
\usepackage{hyperref}
\usepackage{amsmath}
\usepackage{amsthm}
\usepackage{amsfonts}
\usepackage{amssymb}
\usepackage[dvipsnames]{xcolor}
\usepackage{cleveref}
\usepackage{listings}
\usepackage{makecell}
\usepackage[margin = 3cm]{geometry}
\usepackage{enumerate}
\usepackage{mathtools}

\usepackage[style=numeric, labelnumber, defernumbers=true,  backend=biber, url=false, maxbibnames=99]{biblatex}
\usepackage{tikz} 
\usepackage{pgfplots}
\usepackage{pgfplotstable}
\usetikzlibrary{shapes, backgrounds, patterns}
\pgfplotsset{compat=1.18}
\usetikzlibrary{arrows.meta, positioning, calc}
\usepgfplotslibrary{groupplots}

\definecolor{tue_red}{HTML}{C81919}
\definecolor{tue_dark_blue}{HTML}{101073}
\definecolor{tue_blue}{HTML}{0066CC}
\definecolor{tue_cyan}{HTML}{00A2DE}
\definecolor{tue_green}{HTML}{84D200}
\definecolor{tue_yellow}{HTML}{CEDF00}

\definecolor{accent}{gray}{0.95}

\definecolor{color1}{RGB}{0, 121, 178}
\definecolor{color2}{RGB}{255, 124, 37}
\definecolor{color3}{RGB}{37, 160, 55}
\definecolor{color4}{RGB}{220, 32, 44}
\definecolor{color5}{RGB}{147, 104, 186}
\definecolor{color6}{RGB}{143, 85, 76}
\definecolor{color7}{RGB}{230, 119, 192}
\definecolor{color8}{RGB}{127, 127, 127}
\definecolor{color9}{RGB}{192, 188, 55}
\definecolor{color10}{RGB}{0, 191, 206}

\definecolor{orange}{HTML}{E69F00}
\definecolor{cyan}{HTML}{56B4E9}
\definecolor{green}{HTML}{009E73}
\definecolor{yellow}{HTML}{F0E442}
\definecolor{blue}{HTML}{0072B2}
\definecolor{red}{HTML}{D55E00}
\definecolor{purple}{HTML}{CC79A7}

\definecolor{ForestGreen}{RGB}{34,139,34}

\definecolor{cblindred}{RGB}{216, 27, 96}
\definecolor{cblindblue}{RGB}{30, 136, 229}
\definecolor{cblindyellow}{RGB}{255, 193, 7}
\definecolor{cblindgreen}{RGB}{0, 77, 64}

\newtheoremstyle{noparens}
  {0pt}{0pt}   
  {\itshape}    
  {}            
  {\bfseries}   
  {.}           
  { .5em}       
  {\thmname{#1}\thmnumber{ #2}\thmnote{ #3}} 
\theoremstyle{noparens}

\newtheorem{remark}{Remark}[section]
\newtheorem{example}{Example}[section]

\usepackage[framemethod = TikZ]{mdframed}

\newcommand{\makefancybox}[4]{%

    \newenvironment{#1}[1][]{%
        \refstepcounter{#2}%

        \def\mytitle{%
            #3 \csname the#2\endcsname%
            \if\relax\detokenize{##1}\relax%
            \else%
                : ##1%
            \fi%
        }%

        \mdfsetup{
            linecolor = #4,
            linewidth = 0.25pt,
            topline = true,
            bottomline = true,
            leftline = false,
            rightline = false,
            nobreak=true,
            innertopmargin = 4pt,       
            innerbottommargin = 4pt,     
            frametitleaboveskip = 4pt,   
            frametitlebelowskip = 4pt,   
            frametitlerule = false,      
            frametitlebackgroundcolor = #4!20,
            frametitlefont = \bfseries,
            frametitle = \mytitle,
        }

        \begin{mdframed}
    }{%
        \end{mdframed}
    }%

    \crefname{#1}{#3}{#3s}
}

\makefancybox{definition}{definitionT}{Definition}{cblindred}
\makefancybox{lemma}{lemmaT}{Lemma}{cblindblue}
\makefancybox{proposition}{propositionT}{Proposition}{cblindgreen}

\newcommand{\cL}{\mathcal{L}}
\newcommand{\cT}{\mathcal{T}}
\newcommand{\cM}{\mathcal{M}}
\newcommand{\cN}{\mathcal{N}}

\newcommand{\cP}{\mathcal{P}}

\newcommand{\cX}{\mathcal{X}}

\newcommand{\R}{\mathbb{R}}

\newcommand{\dd}{\mathrm{d}}

\DeclareMathOperator*{\argmin}{arg\,min}

\newcommand{\colorV}{\color{black}}
\newcommand{\colorM}{\color{black}}

\newcommand{\colorT}{\color{black}}

\newcommand{\metric}{\rho}

\renewcommand{\u}{{\colorV u}}
\newcommand{\params}{\theta}
\newcommand{\Dp}{{\colorM D(\params)}}
\newcommand{\w}{{\colorM w}}
\renewcommand{\v}{{\colorM v}}

\newcommand{\paramk}{\theta^{(k)}}
\newcommand{\paramkm}{\theta^{(k-1)}}
\newcommand{\paramkp}{\theta^{(k+1)}}

\newcommand{\yk}{y^{(k)}}

\newcommand{\wk}{w^{(k)}}

\newcommand{\glossk}{{\nabla L(\paramk)}}

\newcommand{\vk}{{\colorM v^{(k)}}}
\newcommand{\vkm}{{\colorM v^{(k-1)}}}
\newcommand{\vkp}{{\colorM v^{(k+1)}}}

\newcommand{\dvdt}{\frac{\partial \v}{\partial t}}

\newcommand{\psik}{\psi^{(k)}}
\newcommand{\psikm}{\psi^{(k-1)}}

\newcommand{\psikT}{\psi^{(k)T}}
\newcommand{\psikmT}{\psi^{(k-1)T}}

\newcommand{\psiki}{{\colorM \psi^{(k)}_i}}
\newcommand{\psikj}{{\colorM \psi^{(k)}_j}}

\newcommand{\Qk}{{Q^{(k)}}}

\newcommand{\phik}{{\colorM \phi^{(k)}}}

\newcommand{\pk}{p^{(k)}}
\newcommand{\pkm}{p^{(k-1)}}

\newcommand{\Pk}{{\colorT \cP^{(k)}}}
\newcommand{\Pkm}{{\colorT \cP^{(k-1)}}}

\newcommand{\dPdt}{\frac{\partial \cP}{\partial t}}

\newcommand{\gk}{g^{(k)}}

\newcommand{\Gk}{{G^{(k)}}}

\newcommand{\Ginvk}{{G^{(k)\dagger}}}

\newcommand{\GXk}{{G^{(k)}_X}}

\newcommand{\GXinvk}{{G^{(k)\dagger}_X}}

\newcommand{\GXkkm}{{G^{(k,k-1)}_X}}

\newcommand{\cTk}{{\cT_{\paramk}}}
\newcommand{\cTkm}{{\cT_{\paramkm}}}

\newcommand{\Kmax}{\kappa^{(k)}_{\text{max}}}

\newcommand{\Hk}{{H^{(k)}}}

\newcommand{\dDdpi}{{{\colorT \frac{\partial D}{\partial \params_i}}}}
\newcommand{\dDdpj}{{\colorT \frac{\partial D}{\partial \params_j}}}
\newcommand{\dDdp}{{{\colorT \frac{\partial D}{\partial \params}}}}

\newcommand{\GradP}{\nabla}
\newcommand{\LossGrad}{\GradP L}
\newcommand{\LossVGrad}{{\colorV \nabla \cL}}

\newcommand{\gradM}{\operatorname{grad}_{\cM}}

\newcommand*\samethanks[1][\value{footnote}]{\footnotemark[#1]}

\title{Natural gradient descent with momentum}
\author{Anthony Nouy\thanks{Centrale  Nantes, Nantes Université,  Laboratoire de Mathématiques Jean Leray UMR CNRS 6629 (anthony.nouy@ec-nantes.fr and agustin.somacal@ec-nantes.fr)}~ and Agustin Somacal\samethanks[1]}
\date{}

\begin{document}

\maketitle

\begin{abstract}
We consider the problem of approximating a function by an element of a nonlinear manifold which admits a differentiable parametrization, typical examples being neural networks with differentiable activation functions or tensor networks.
    Natural gradient descent (NGD) for the optimization of a loss function can be seen as a preconditioned gradient descent where updates in the parameter space are driven by a functional perspective. In a spirit similar to Newton’s method, a NGD step uses, instead of the Hessian, the Gram matrix of the generating system of the tangent space to the approximation manifold at the current iterate, with respect to a suitable metric. This corresponds to a locally optimal update in function space, following a projected gradient onto the tangent space to the manifold.
    Still, both gradient and natural gradient descent methods  get stuck in local minima. Furthermore, when the model class is a nonlinear manifold  or the loss function is not ideally conditioned (e.g., the KL-divergence for density estimation, or a  norm of the residual of a partial differential equation in physics informed learning), even the natural gradient might yield non-optimal directions at each step. This work introduces a natural version of classical inertial dynamic methods like Heavy-Ball or Nesterov and show how it can improve the  learning process when working with nonlinear model classes.
\end{abstract}

\section{Introduction}

Optimization strategies based on gradient descent (GD) have become the standard method to train models in modern machine learning applications. Since its first introduction by \cite{cauchyMethodeGeneralePour1847} many improvements have been made to understand its theoretical capabilities and limits, accelerate its convergence rates or lower its computational cost. However, GD looks at the optimization process from the parameters perspective disregarding the fact that ultimately it is the functions they instantiate that we want to update in the direction of steepest descent. To correct this  bias in the GD update \cite{amariNeuralLearningStructured1996,amariNaturalGradientWorks1998} proposed to take into account the geometry of the model class (the approximation manifold) introducing the notion of natural gradient descent (NGD). Since then NGD has been proven to be effective in multiple domains: online learning, blind source separation as in the seminal works of \cite{amariNeuralLearningStructured1996,amariNaturalGradientWorks1998}, in reinforcement learning \cite{petersNaturalActorCritic2008}, deep neural network training \cite{parkAdaptiveNaturalGradient2000,martensOptimizingNeuralNetworks2020} and more recently in physics informed learning \cite{müllerAchievingHighAccuracy2023,schwenckeANaGRAMNaturalGradient2024,jniniGaussNewtonNaturalGradient2024}, and  dynamical low rank approximation \cite{bonStableNonlinearDynamical2025a}.

Although the NGD update gives the steepest direction of descent from the functional perspective this is only true locally as once we take a finite step we might deviate quite a lot from the optimal update due to the nonlinearities of the model class or the loss being other than induced by a Hilbertian norm. Furthermore, when the gradient at each iteration is estimated using random samples the estimation errors might amplify even more the deviation from the expected direction. Finally, NGD as GD will also get stuck in the first local minima encountered as no inertia is present in the dynamics. We claim that adding information from previous iterations can effectively address to some extent these issues as they can be used to correct the direction, reduce the variance in the presence of stochasticity and add an inertial term allowing the optimization to escape local minima.

Recently, there have been attempts to incorporate momentum into first order optimizers in the context of Riemannian manifolds \cite{kimNesterovAccelerationRiemannian2022,brantnerGeneralizingAdamManifolds2023,ahnNesterovsEstimateSequence2020,zhangRiemannianAcceleratedGradient2018}. Many of these works define the accelerated dynamics using the exponential and logarithmic maps between the tangent space and the manifold under the rationale that they yield the exact dynamics which the geodesic descent curve suggests. However, in doing so, many of the algorithms become intractable, computationally too expensive or explicit only in the case of very particular manifolds like the Stiefel manifold of orthogonal matrices \cite{absilOptimizationAlgorithmsMatrix2008}. We argue that forcing the dynamics to follow exactly the geodesic is too strong a constraint given that the optimization will be done on successive iterations correcting eventually the deviations. Thus, relaxing this hard constraint we can still get accelerated methods that have comparative computational complexity with NGD and can work in broader situations.

The paper is organized as follows. In \Cref{sec:setting}, we introduce the problem setting and notations. In \Cref{sec:gradient-descent}, we present classical gradient algorithms—gradient descent and Newton's method—and show their relationship with NGD. In particular, in \Cref{sec:flow}, we formulate the functional interpretation of the natural gradient. In \Cref{sec:natural-momentum}, we present various natural versions of momentum strategies. In \Cref{sec:classical-momentum}, we recall classical momentum strategies. In \Cref{sec:heavy-ball}, we show how a natural version of the heavy-ball algorithm can be devised by discretizing the gradient flow  in function space. In \Cref{sec:nesterov}, we propose two variants of a natural Nesterov method. In \Cref{sec:numerics}, we present four numerical examples that showcase the improvements and limitations of the proposed functional versions of momentum methods.
\newline

In \url{https://codeberg.org/akuasoma/Opt4Fun} is hosted the code to reproduce the numerical experiments.

\section{Problem setting}
\label{sec:setting}

We consider a Banach space $V$ of  functions with norm $\Vert \cdot \Vert_V$. In the case of a Hilbert space, we let $({\cdot,\cdot})_V$  define the inner product.
The objective is to approximate a target function $\u$ from $V$ by an element  of a  parametrized model class
$$\cM\coloneqq\{\v=\Dp : \theta \in \R^d\} \subset V,$$
where
$D: \R^d \to V$ is a differentiable map. The  partial derivative $\dDdpi(\theta)$ of $D$ with respect to the $i$-th parameter $\theta_i$, evaluated at $\theta$,  is a function in $V$ that we will write $\psi_i(\theta)$.
The functions $\psi_1(\theta),\dots, \psi_d(\theta)$ generate the tangent space $\cT_{D(\theta)} \subset V$ of the manifold $\cM$ at point $\Dp$ defined by
$$
\cT_{D(\theta)}\mathcal{M} = \operatorname{span} \left\{\psi_i(\theta) =  \dDdpi(\theta) : 1\le i \le d \right \} =: \cT_\theta.
$$
In what follows it will be convenient to denote $$ {\colorT\frac{\partial D}{\partial \theta}}(\theta) \coloneqq\psi(\theta) \coloneqq(\psi_i(\theta))_{1\le i\le d}\in V^d$$
as the tuple of functions such that  an element ${\colorT w} \in \cT_{\theta}$ can be written as
$$
{\colorT w}= \psi(\theta)^T c =\sum_{i=1}^d c_i \psi_i(\theta),
$$
with $c\in\R^d$ the coordinates of $w$ in the generating system of $\cT_{\theta}$.

\vspace{10pt}
\begin{example}
    The  {model class} $\cM$, which is determined by $D$, can be for example:
    \begin{itemize}
        \item a linear subspace of $V$ when $D$ is a linear map from $ \mathbb{R}^d$ to $V$, in which case
        $\Dp=\sum_{j=1}^{d} \theta_j {\colorT \psi_j}$, where the ${\colorT \psi_j} = \dDdpj$ form a basis of  $\cM$. Here the tangent space at any point is equal to the manifold itself.
        \item a shallow neural network defined on $\mathbb{R}^n$, with $k$ neurons, in which case
        $$\Dp(x) = a^T\sigma(A x + b) + c$$
        with $\sigma: \R \to \R$ a differentiable nonlinear activation function (applied entry-wise to a vector), $A\in\R^{k\times n}$ and $a\in\R^{ k}$ the weights, and $b\in\R^{k}$ and $c \in \mathbb{R}$ the biases. The parameter $\theta = \operatorname{vec}(A, a, b, c)$ gathers    the set of weights and biases, so that $d=(n+2)k+1$ and the tangent space at $\theta$ is generated by
        \begin{align*}
        &{\colorT\frac{\partial D}{\partial c}}(\theta)(x)=1,  &{\colorT\frac{\partial D}{\partial a_i}}(\theta)(x)&=\sigma(b_i+A_i^T x),\\
        &{\colorT\frac{\partial D}{\partial b_i}}(\theta)(x)=a_i\sigma'(b_i+A_i^T x),  &{\colorT\frac{\partial D}{\partial A_{ij}}}(\theta)(x)&=a_i\sigma'(b_i+A_i^T x)x_j,
        \end{align*}
        where $A_i^T$ denotes the $i$-th row vector of $A$.
    \end{itemize}
\end{example}

We assume that the target function $\u$ is a  minimizer over $V$ of a
\emph{loss function} $\cL: V\to \R$ defined as
$$\cL(\v)=\int_\cX \ell( \v , x) \dd \mu(x),$$
where $\ell:   V \times \cX \to \R$ is a \emph{point-wise loss} and
 $\cX$ is a set equipped with a measure $\mu$. Note that in most examples, functions in $V$ will be defined on the set $\cX$, but they may be defined on other sets as well.

The \emph{excess loss}  $\cL(\v) - \cL(\u) $ provides a measure of discrepancy between an  approximation $\v$ and the target function $\u$.
The approximation task then consists in  finding ${\colorM v^*} \in \cM$ that minimizes the loss, i.e.
\begin{equation}
    {\colorM v^*} \in \underset{\v\in \cM}{\argmin} \, \cL(\v). \label{min-loss}
\end{equation}

We assume that $\cL$ is Fréchet differentiable, and we denote by $\cL'(\v) : V \to \mathbb{R}$ its Fréchet differential at $\v$. If $V$ is a Hilbert space we can identify $\cL'(\v)$ with the gradient $\LossVGrad(v)\in V$ at $v$ by Riesz representation so that $\cL'(\v)(w) = (\LossVGrad(\v),w)_V$ for any $w\in V$. We denote by $\mathcal{R}_V :V'\to V$ the Riesz map such that $\mathcal{R}_V\cL'(v)=\LossVGrad(\v)$.

Under some integrability and regularity  assumptions  on    the  point-wise loss $\ell$  (see \cite{kammarfrechetdiff2016}), the Leibniz rule
yields
 \begin{align}
\cL'(\v)(\w) &= \int_\cX \ell'(\v, x)(\w) \,  \dd \mu(x) ,
\end{align}
with $\ell'$  the differential of $\ell$ with respect to the first variable.

\vspace{10pt}
\begin{example}\label{ex:examples_loss}
We present several examples of
      {loss} functions $\cL$ and corresponding  {point-wise loss} functions $\ell$.
    \begin{enumerate}[(i)]
        \item Assuming $\ell(\v , x) = \tilde \ell(\v(x) , x)$ with $\tilde \ell : \mathbb{R} \times \cX \to \mathbb{R}$ (i.e. the point-wise loss at $x$ only depends on the  evaluation of $\v$ at $x$), then $\ell'(\v , x) (\w) = \tilde \ell'(\v(x) , x) \w(x)$, so that $\nabla \cL(\v)$ is identified with the function $x\mapsto  \tilde \ell'(\v(x) , x)$ in $V= L^2_\mu(\cX)$, and
\begin{align}
\cL'(\v)(\w)=(\nabla \cL(\v) , \w)_V &= \int_\cX  \tilde \ell'(\v(x), x) \w(x) \,  \dd \mu(x) .
\end{align}
This includes the classical least-squares regression setting with $\tilde \ell(\v(x) , x) = \frac{1}{2}(\v(x) - \u(x))^2$ and $\tilde \ell'(\v(x) , x) = \v(x) - \u(x).$
    \item Assuming $\ell(\v , x) = \frac{1}{2} \Vert L_x \v - L_x \u \Vert_2^2$, where $L_x : V \to \mathbb{R}^p$ is a family of operators indexed on $\cX$, and $V$ a Hilbert space  equipped with the norm
      $\Vert \v \Vert_V^2 = \int_\cX \Vert L_x \v \Vert_2^2 \, \dd \mu(x)$, then $\cL(\v) =\frac{1}{2} \Vert \v - \u \Vert_V^2$, $\cL(u) = 0$,  $\nabla \cL(\v) = \v - \u $, i.e.
      $$
    \cL'(\v)(\w)=(\nabla \cL(\v) , \w)_V = \int_{\cX} L_x (\v - \u)^T L_x(\w) \,  \dd \mu(x) .
      $$
      Typical cases include $V = L^2_\mu(\cX)$ with $L_x \v = \v(x)$, which is again the classical least-squares regression setting, or $V = H^1_\mu(\cX)$ with $L_x \v = (\nabla \v(x)^T \;  \v(x) )^T \in \mathbb{R}^{p+1}$, which is a regression problem using evaluations of a function and its derivatives.

        \item For least-squares density estimation, with $\mu$ the target distribution over $\cX$, with density $u$ with respect to a reference $\nu$, assuming $\u\in V =L^2_\nu(\cX) $,
        $$\ell(\v , x) = \frac{1}{2} \Vert \v \Vert_{L^2_\nu}^2  -  \v(x), \quad \cL(\v) = \frac{1}{2}\Vert \v \Vert_{L^2_\nu}^2 - \int_\cX \v(x) \dd\mu(x).$$
         The gradient of $\cL$ is $\LossVGrad(\v) = \v - \u$ and the excess loss $\cL(\v) - \cL(\u) =  \frac 1 2 \Vert \v - \u \Vert_{L^2_\nu}^2$.
        \item For density estimation using cross-entropy, with $\mu$ the target distribution over $\cX$, with density $\u$ with respect to a reference $\nu$,
        $$\ell(\v, x) = -\log(\v(x)) , \quad \cL(\v) = - \int_\cX \log(\v(x))  \dd\mu(x).$$
        Here  $\cL(\v)$ is the cross-entropy of $\v$ with respect to $\u$ and the excess loss
        $$\cL(\v) - \cL(\u)=\int_\cX  \log\left(\frac{\u(x)}{\v(x)}\right) \u(x)\dd \nu(x) = D_{KL}(\u || \v),$$
        which is the Kullback-Leibler divergence between $\u$ and $\v$. It holds that
        $$\cL'(v)(w)=\int -w(x)\frac{u(x)}{v(x)} \dd\nu(x).$$
        Assuming $u/v$ is in $L^2_\nu$, $\cL'(v)$ can be identified with the gradient $\nabla_{L^2_\nu}\cL(\v)(x) = - \frac{\u(x)}{\v(x)}$ in $L^2_\nu$.
        \item For the solution of an elliptic PDE $- \nabla\cdot (a(x) \nabla u(x) ) = f(x)$ on a bounded  domain $\cX$, with $a$ positive, uniformly bounded and   uniformly bounded away from zero, with homogeneous Dirichlet boundary conditions,
        we consider $V = H^1_0(\cX)$ equipped with the norm $\Vert \v \Vert_V^2 = \int_{\cX} a(x) \Vert \nabla \v(x)\Vert_2^2 \, \dd x$, the point-wise loss
        $$
        \ell(\v ,x) = \frac{1}{2} a(x) \Vert \nabla \v(x)\Vert_2^2 - f(x) \v(x),
        $$
        and the loss
        $$\cL(\v) = \frac{1}{2} \Vert v \Vert_V^2 - \int_\cX f(x) \v(x) \dd x.
        $$
         Here $\cL(\v) - \cL(\u) = \frac{1}{2} \Vert \u- \v \Vert_V^2$, and $\nabla \cL(\v) = \v- \u$.

        \item For the solution of a PDE $A(u)(x) = f(x)$ defined on a bounded domain $\cX$, with boundary conditions, with $A$  a partial differential operator and $f \in L^2(\cX)$, a minimal residual loss is defined by
        $$
         \ell(\v , x) = \vert A(\v)(x) - f(x) \vert^2 , \quad   \cL(\v) = \int_\cX \vert A(\v)(x) - f(x) \vert^2 \dd x,
        $$
        where we assume that
         $A(\u) \in L^2(\cX)$. Letting $V$ be the space of functions $\v$ such that $A(\v) \in L^2(\cX)$ and satisfying boundary conditions, it holds that $\u$ minimizes $\cL$ over $V$ and $\cL(\u) = 0$.
         For linear operators, we can introduce the indexed family of operators $L_x : V\to \mathbb{R} $ such that $L_x \v = A(v)(x)$, and we recover the setting introduced in (ii), with $f(x) = L_x \u(x)$.
         We can also reformulate the problem as a least-squares approximation of $f$ in  $L^2(\cX)$, with a loss defined on $\tilde V = L^2(\cX)$ by
         $$
         \tilde \cL(w) = \int_\cX (w(x) - f(x))^2 \dd x.
         $$
          Approximating $\u$ in the parametrized manifold $\cM \subset V$ corresponds to approximating   $f$ in the operator induced set $\tilde \cM \coloneqq \{w = \tilde D(\theta) : \theta \in \Theta\},$ with $\tilde D = A\circ D$. This reformulation will be used in numerical experiments.
    \end{enumerate}
\end{example}

We then define the parameter-dependent loss  $L: \Theta \to \R$ as the composition
$$
L(\theta) \coloneqq \cL ( D(\theta)),
$$
and  $\nabla L(\theta) \in \R^d$  the gradient of the loss with respect to the parameters, evaluated at $\theta$.
Applying the chain rule we have that
\begin{align}
 \LossGrad(\theta)_i &= \cL'(D(\theta))({\colorT\psi_i}(\theta)), \forall i=1,\dots,d,
\end{align}
which we denote $$\LossGrad(\theta) = \cL'(D(\theta))(\psi(\theta)).$$

In  practice, the loss function is  estimated from available samples or through numerical integration, which leads us to consider an \emph{empirical loss}
\begin{align*}
       \cL_m(\v) \coloneqq  \sum_{j=1}^m w_j\ell(\v , x_j) = \int_{\cX} \ell(\v , x)\, \dd\mu_m(x),
\end{align*}
where the $x_j$ and $w_j$ are integration points and integration weights respectively, and $\mu_m = \sum_{j=1}^m w_j \delta_{x_j}$. The corresponding parameter-dependent empirical loss $L _m(\theta) \coloneqq \cL_m(D(\theta))$ has gradient
$$
\nabla L_m(\theta)  =  \sum_{j=1}^m w_j   \ell'(D(\theta) , x_j) \big(\psi(\theta)\big) = \int_\cX \ell'(D(\theta),x)(\psi(\theta)) \dd\mu_m(x), \quad 1\le i \le d.
$$
Typically, an importance sampling  Monte-Carlo integration consists in drawing the points $x_j$ according to a distribution $w(x)^{-1} \dd \mu(x)$, with $w^{-1} $ some probability density function, and taking $w_j = \frac{1}{m} w(x_j)$. In an active learning setting, the distribution can be adapted to the approximation task. In particular, it can be adapted to
the tangent space $\mathcal{T}_\theta$  using Christoffel sampling or some variant (see \cite{gruhlkeOptimalSamplingStochastic2024,Adcock2024Sep,Nouy2025}).
Note that once the loss $\cL$ and the model's parametrization map $D$ have been defined, the gradients can be evaluated in practice  using automatic differentiation.

\section{Gradient descent and beyond}
\label{sec:gradient-descent}

The optimization problem \eqref{min-loss} is equivalent to an optimization problem over the parameter space
$$\theta^\star \in \underset{\theta\in \Theta}{\argmin} \, L(\theta),$$
with corresponding solution $v^\star = D(\theta^\star)$ in the function space.

To define a gradient descent algorithm, we can consider a first order Taylor expansion of the loss at the current iterate  $\paramk $ and restrict the next step $\paramkp$ to belong to a certain neighborhood of $\paramk$. This can be achieved by introducing a penalized loss function
\begin{equation}
	  L^{(k)}(\theta) \coloneqq L(\paramk) + \LossGrad(\paramk)^T (\theta-\paramk) + \frac{1}{2s} \metric(\theta, \paramk),
	  \label{eq:taylor_gd}
\end{equation}
where $\metric: \mathbb{R}^d\times\mathbb{R}^d\to \R^+$  is associated with some metric on the parameter space and $s>0$. We then define $\paramkp$ as the minimizer of  $L^{(k)}(\theta)$. The first order optimality condition is
\begin{equation}
	  \GradP \metric(\paramkp, \paramk) = -2s \LossGrad(\paramk),
	  \label{eq:gdrule}
\end{equation}
where $\GradP \metric(\theta , \paramk) $ is the gradient of $\theta\mapsto \metric(\theta , \paramk)$ evaluated at $\theta$.
If $\metric(\theta, \paramk) = \|\theta -\paramk\|^2_{2}$ we have   $\GradP \metric(\theta, \paramk)=2(\theta-\paramk)$ and we recover the classical gradient descent update rule
\begin{equation}
\paramkp=\paramk -s \LossGrad(\paramk)
\label{eq:GD}
\end{equation}
with $s$ being the step size. This can be seen as the discretization of a gradient flow in parameter space
\begin{equation}
\frac{\partial \theta}{\partial t} = - \nabla L(\theta(t)).\label{eq:gradient-flow}
\end{equation}

\subsection{Preconditioned gradient descent and Newton's method}

Given some symmetric positive-definite matrix $M\in\R^{d\times d}$, choosing
\begin{equation}
	  \metric(\theta, \paramk) = \|\theta-\paramk\|^2_{M} = ( \theta-\paramk, \theta-\paramk )_M = (\theta-\paramk)^TM(\theta-\paramk),
	  \label{eq:precondgrad}
\end{equation}
it holds $\GradP \metric(\theta, \paramk)=2M(\theta-\paramk)$, and we obtain
a preconditioned update rule
\begin{equation}
	  \paramkp = \paramk-s M^{-1} \LossGrad(\paramk).
	  \label{eq:precond}
\end{equation}

Letting $H_L (\theta) : \R^d \times \R^d \to \R$ be the Hessian of $L$ evaluated at $\theta$, a Newton iteration consists in choosing for $M$ the Hessian matrix $H_L(\paramk)$ at $\paramk$, with entries
$$
M_{ij} = H_L(\paramk)(e_i, e_j),
$$
where the $e_i$ are the canonical vectors in  $\R^d$.

\subsection{From Newton's method to natural gradient descent}
\label{sec:gauss-newton}

The Hessian $H_L(\theta) $ of the loss $L$ at $\theta$ can be expressed in terms of
the Hessian $H_{\cL}(D(\theta)) : V \times V \to \R$  of the functional loss $\cL$ at $D(\theta)$, and the Hessian $H_D(\theta) : \R^d \times \R^d \to V$ of the map $D : \R^d\to V$ at $\theta$.  Indeed,
for any
$w\in \R^d$, letting $\psi(\theta) = \dDdp(\theta)$, it holds
\begin{align*}
L(\theta + w)& = \cL( D(\theta + w)) \\
&=\cL( D(\theta) + \psi(\theta)^T w + \frac{1}{2} H_D(w,w) + o(\Vert w\Vert_2^2)) \\
&= \cL( D(\theta)) + \cL'(D(\theta))(\psi(\theta)^T w) + \frac{1}{2} H_{\cL}(\psi(\theta)^T w,\psi(\theta)^T w) + \frac{1}{2} \cL'(D(\theta))(H_D(w,w)) + o(\Vert w\Vert_2^
2)
\end{align*}
so that, by identification,
$$
H_L(\theta)(w,w) =  H_{\cL}(D(\theta))(\psi(\theta)^T w,\psi(\theta)^T w) +    \cL'(D(\theta))(H_D(w,w)) .
$$
From the above development we see that the Hessian matrix $H_L(\theta)$ at $\theta$ used for a  Newton's iterate can be written
\begin{align}
H_L(\theta) = G(\theta) + K(\theta),
 \label{eq:hessian}
\end{align}
where
\begin{itemize}
     \item
$G(\theta) \in \R^{d\times d}$ is a symmetric positive semi-definite matrix with entries
$$
G(\theta)_{ij} = H_\cL(D(\theta))(\psi_i(\theta) , \psi_j(\theta))
$$
which corresponds to the functional Hessian  $H_\cL(D(\theta))$ at point $D(\theta)$ applied to the elements of the generating system of the tangent space $\cT_{\theta}$,
\item  $K(\theta) \in \R^{d\times d}$ is a symmetric matrix with entries
$$
K(\theta)_{ij} =  \cL'(D(\theta))(H_D(\theta)(e_i,e_j)).
$$
The matrix $K(\theta) $ is zero whenever the manifold $\cM$ is a linear space or has zero curvature at $\theta$, since in this case $H_D(\theta)=0$.
\end{itemize}

If the approximation manifold is nonlinear but has a uniformly bounded and small curvature, one could disregard the second term $K$ in \eqref{eq:hessian}, and consider the iteration \eqref{eq:precond} with the preconditioner $M = G^{(k)} \coloneqq G(\paramk).$
In the case where $\cL(v) = \frac{1}{2}\Vert u - v \Vert_V^2$ and $V$ is a Hilbert space, $H_\cL$ coincides with the inner product in $V$ and $G^{(k)}$ is the Gram matrix of the generating system of $\cT_{\paramk}$ with respect to the inner product in $V$. This corresponds to a Gauss-Newton  iteration in function space \cite{martensNewInsightsPerspectives2020}, which consists in locally linearizing the manifold.

This also corresponds to a natural gradient descent (NGD) where the preconditioner $M$ is the  Gram matrix of the generating system of $\cT_{\paramk}$ for the inner product in $V$.
For a  general loss functional, the approximate Newton step also
corresponds to a natural gradient descent  where the preconditioner
$ M =  G^{(k)}$ is the Gram matrix
of the generating system of $\cT_{\paramk}$ with respect to the (semi-)inner product induced by
the functional Hessian $H_\cL$ at $D(\paramk)$.
The different variants of NGD found in the literature come from different choices of the inner product on $\cT_{\paramk}$, which may not necessarily correspond to the one induced by $H_\cL$ or the inner product of $V$ when $V$ is a Hilbert space.

\begin{remark}
Although in some cases an NGD update might coincide with a second order method like the Newton method it is important to note that for computing $G(\theta)$ one only needs first order information whereas for a second order method, second order derivative $H_D(\theta)$ of the parameterization map are needed, thus adding an extra computational complexity.
\end{remark}

\subsection{Natural gradient}
\label{sec:nat-grad}

The Euclidean metric in the parameter space induces a metric in the tangent space $\cT_k := \cT_{\paramk}$ which may be not suitable for the optimization task in the function space. A natural gradient descent  consists in introducing on the tangent space $\cT_{k}$ a more natural metric from the functional perspective. Given an inner product $(\cdot, \cdot)_{W}$ on $\cT_{k}$, with associated norm $\Vert \cdot \Vert_{W}$, we define
$$
\rho(\theta , \paramk) =  \Vert \psikT (\theta - \paramk) \Vert_{W}^2 = (\theta - \paramk)^T \Gk (\theta - \paramk),
$$
with $\psik\coloneqq \psi(\paramk)$ and
\begin{equation}
\Gk \coloneqq (\psik ,\psikT)_W
\label{eq:gram}
\end{equation}
the Gram matrix of the generating system of $\cT_{k}$ with respect to the inner product $(\cdot,\cdot)_W$ with entries $G^{(k)}_{ij} = (\psiki ,\psikj )_W$.
The matrix $G^{(k)}$ may not be invertible. This occurs when the elements of the generating system $\psik$ are linearly dependent ($\dim(\cT_k) < d$), in particular for over-parametrized manifolds. If $G^{(k)}$ is not invertible, the convex functional \eqref{eq:taylor_gd} admits infinitely many minimizers. A gradient step can be defined by selecting the minimizer with minimal euclidean norm, this corresponds to a gradient step
\begin{equation}
\paramkp = \paramk - s \Ginvk \nabla L(\paramk), \label{eq:ngd-update}
\end{equation}
with $\Ginvk$ being the Moore-Penrose pseudo-inverse of $G^{(k)}$.
Given a spectral decomposition $\Gk = U \Lambda U^T = \sum_{1\le i \le d} \lambda_i u_iu_i^T$, with $U$ an orthogonal matrix and $\Lambda$ a diagonal matrix, the Moore-Penrose pseudo-inverse is given by
\begin{equation}
\Ginvk =  \sum_{1\le i \le d, \lambda_i \neq 0} \lambda_i^{-1} u_i u_i^T =: U_+\Lambda_+^{-1}U_+^T.
\label{eq:penrose}
\end{equation}
Defining the matrix
$\Qk \coloneqq \Lambda_+^{-\frac 1 2}U_+^T$
we have that $\Ginvk=\Qk^T\Qk$ and  $\phik\coloneqq \Qk \psik  \in W^{\dim(\cT_k)}$
forms a $W$-orthonormal basis of $\cT_k$.
Depending on the choice of norm $\Vert \cdot \Vert_W$ on  the tangent space $\cT_k$, one will find different expressions for the matrix $G(\theta)$.

\vspace{10pt}
\begin{example}

We consider the same cases as in  Example \ref{ex:examples_loss}.
\begin{enumerate}[(i)]
\item Given a point-wise loss $\ell(v , x) = \tilde \ell(v(x) , x)$, and  assuming $\tilde \ell( \cdot , x)$ is two times differentiable, we can equip the tangent space $\cT_{\theta}$ with the metric induced by the Hessian $H_\cL$ at $v = D(\theta)$, and
$$G^{(k)}_{ij}=(\psiki, \psikj)_W = \int_{\cX} \psiki(x) \tilde \ell''(v(x) , x) \psikj(x) \dd\mu(x). $$
\item Assuming $\cL(v) = \frac{1}{2} \Vert u-v \Vert_V^2$ with  $\Vert v \Vert_V^2 = \int_{\cX} \Vert L_x v \Vert_2^2 \dd\mu(x)$,  a natural choice is $W= V, $ and
$$
G^{(k)}_{ij} = \int_\cX (L_x \psiki)^T L_x \psikj \, \dd\mu(x).
$$
\item For least-squares density estimation in $V = L^2_\nu$, a natural choice is $W=V$.
\item For density estimation with cross-entropy, we can equip the tangent space $\cT_{\theta}$ with the metric induced by the Hessian $H_\cL$ at $v = D(\theta)$, that yields
$$
G^{(k)}_{ij} = \int_\cX \psiki(x) \psikj(x) \frac{u(x)}{v(x)^2} \dd\mu(x)
$$
where $u$ is the target density with respect to the measure $\nu.$
\item For the solution of the linear elliptic PDE $- \nabla \cdot (a(x) \nabla u) = f(x)$ in $V = H^1_0(\cX)$, we can choose $W=V$, the space equipped with the operator induced norm $\Vert v \Vert_V^2 = \int_{\cX} a(x) \Vert \nabla v \Vert_2^2 \dd\mu(x)$. The matrix
$$
G^{(k)}_{ij} = \int_\cX a(x) \nabla \psiki  \cdot \nabla \psikj \dd\mu(x)
$$
corresponds to the standard ``stiffness'' matrix and $\psikT\Ginvk\glossk$ is the Galerkin projection of $\nabla \cL(D(\paramk)) = D(\paramk) - u$ onto $\cT_{\paramk}$.
\end{enumerate}
\end{example}

The actual computation of the NGD update given by \eqref{eq:ngd-update} may require some form of regularization when the matrix $G^{(k)}$ has a bad effective condition number (ratio of maximal eigenvalue and minimal nonzero eigenvalue). Typical choices are:
\begin{itemize}
    \item \textit{Spectral cutoff}: we replace $[G^{(k)}]^\dagger$ with $[G^{(k)}]^\dagger_\varepsilon = \sum_{1\le i \le d, \lambda_i > \varepsilon} \lambda_i^{-1} u_i u_i^T$.
    \item \textit{Spectral shift} or \textit{Tikhonov regularization}: we consider a metric $\rho(\theta , \paramk) = (\theta-\paramk)^T (G^{(k)} + \varepsilon I)(\theta - \paramk)$, with $\epsilon>0$, which yields   an update rule
    $$
    \paramkp = \paramk - s (G^{(k)} + \epsilon I)^{-1} \nabla L(\paramk).
    $$
    This tends to a natural gradient step when $\epsilon \to 0$, while it tends to a classical gradient step when $\epsilon \to \infty$ and $s = s(\epsilon) \sim  \epsilon.$
    \item \textit{Spectral flooring}: a combined approach where $[G^{(k)}]^\dagger$ is replaced by $[G^{(k)}]^\dagger_\varepsilon + \Gamma_{\varepsilon^{-1}}$ with
    $$\Gamma_{\varepsilon^{-1}}\coloneqq \sum_{1\le i \le d, \lambda_i \leq \varepsilon} \varepsilon^{-1} u_i u_i^T.$$
    This method splits the parameter correction in two orthogonal parts, one that will follow the natural gradient direction projected onto $\mathrm{span}\{ u_i:1\le i \le d, \lambda_i > \varepsilon\}$ and another following the projection of the classical gradient onto the space $\mathrm{span}\{u_i:1\le i \le d, \lambda_i \leq \varepsilon\}$. This method allows exploration of a space that spectral cutoff can not  explore  while simultaneously avoiding the spectral bias created by the spectral shift method.
\end{itemize}

\subsection{Natural gradient in function space}
\label{sec:flow}

In order to interpret the NGD in function space, we first note that
$$
g^{(k)} := \psikT \Ginvk \nabla L(\theta^{(k)}) \in \cT_k
$$
is the solution of
$
(g^{(k)} , \psik)_W = \nabla L(\paramk) = \cL'(\vk)(\psik)
,$
with $\vk = D(\paramk),$ or equivalently
$$
(g^{(k)} , w)_W = \cL'(\vk)(w) \quad \forall w \in \cT_k.
$$
Thus $$g^{(k)} = \gradM \cL(\vk) \in \cT_k$$ is the Riemannian gradient of $\cL$ at $\vk$, which is the Riesz representer of $\cL'(\vk)$ in the tangent space  $\cT_{\vk} = \cT_k  $ equipped with the inner product $(\cdot,\cdot)_W$. The natural gradient therefore operates in the parameter space a correction in the direction of the coordinates of the functional gradient in the generating system $\psik$ of $\cT_{\vk}$. More precisely, the NGD iteration \eqref{eq:ngd-update} corresponds to the following iteration in function space
$$
v^{(k+1)} = R(\vk - s \gradM \cL(\vk))
$$
with $R$ being a natural retraction defined by
\begin{equation}
R(v + \Delta) = D(\theta + \delta), \quad \text{for} \quad  v = D(\theta)\in \cM \quad \text{and} \quad  \Delta  = \psi(\theta)^T \delta \in \cT_v . \label{eq:retraction}
\end{equation}
With this definition, we have
\begin{equation*}
v^{(k+1)} = D(\theta^{(k+1)}), \quad \theta^{(k+1)} = \theta^{(k)} + s p^{(k)}
\end{equation*}
where $p^{(k)}$ is such that
$$
\psikT p^{(k)} = - \gradM \cL(\vk).
$$
One iteration of the algorithm in function space is depicted on \Cref{fig:nat-grad-0}.

\begin{figure}[h]
 \centering
\begin{tikzpicture}[>=Stealth, line width=0.8pt,
    declare function={
        manifold(\x) = -0.5*(\x-2)*(\x+2);
    }
    ]
    \pgfmathsetmacro{\px}{1.15};
    \pgfmathsetmacro{\py}{manifold(\px)};
    \pgfmathsetmacro{\slope}{-\px };

    \draw[red] plot[domain=-1.5:2, samples=50] (\x, {manifold(\x)}) node[left] {$\color{red} \cM$};

    \pgfmathsetmacro{\w}{1.4}
    \pgfmathsetmacro{\pr}{1.0}
    \draw[ForestGreen, thick] (\px-\w,\py-\slope*\w) -- ({\px+\w}, {\py+\slope*\w}) node[right] {$\cT_{v^{(k)}}$};
    \draw[->, black, dashed, thick] (\px,\py) -- ({\px-2*\pr}, {\py-\slope*2*\pr}) node(projection) [label=above left:{$-\gradM   \mathcal{L}(v^{(k)})$}] {};

    \draw[->, ForestGreen, thick] (\px,\py) -- ({\px-1.5*\pr}, {\py-1.5*\slope*\pr}) node(update) [label=left:{${\color{ForestGreen}-s\gradM   \mathcal{L}(v^{(k)})}$}] {};

    \fill[red] (\px,{manifold(\px)}) circle (2pt) node(v) [label=below left:{${\color{red} v^{(k)}}$}] {};
    \fill[red] (\px-2,{manifold(\px-2)}) circle (2pt) node(vp) [label= left:{${\color{red} v^{(k+1)}}$}] {};


    \draw[->, thick, dotted, orange] (update.center) to[bend right=0] node[midway, orange, left] {$R$} (vp.center);

\end{tikzpicture}

    \caption{Diagram of a natural gradient step.
    The retraction $R$ takes the proposed update ${\colorT -s \gradM \cL(\vk)} = s\psikT \pk$ in the tangent space $\cT_{\vk}$ and yields the next iterate ${\colorM \vkp}=D(\paramkp)=D(\paramk + s\pk)$.}
    \label{fig:nat-grad-0}
\end{figure}

\begin{remark}
For differentiable manifolds with known
exponential and logarithmic maps (e.g., low-rank tensor manifolds), the retraction defined by them might be a preferable choice, except if they are   too expensive to compute.
\end{remark}

If the manifold is embedded in the  Hilbert space $V$ and we choose $W=V$, or if $W \subset V$, we have the following interpretation of the gradient as a projection of an element in $W$. Figure \ref{fig:nat-grad} illustrates one step of natural gradient in this setting.
\begin{lemma}
\label{lem:projection}
Assume $W\subset V$. Then for all $v\in \cM$, $\cL'(v)$ is  a linear form on $W$ and the restriction $\cL'(v)_{\vert W}$ can be identified by Riesz representation with $\nabla _W \cL(v) \in W$. Thus
$$
\gradM \cL(v) = P^W_{\cT_{v}} \nabla_W \cL(\v),
$$
with $ P^W_{\cT_{v}} $ the $W$-orthogonal projection onto $\cT_v.$ 
In particular, for $V$ a Hilbert space and $W=V$, $$\gradM \cL(v) = P^V_{\cT_{v}} \nabla \cL(\v).$$
\end{lemma}
\begin{proof}
By definition of  $\gradM \cL(v)$ and  $\nabla_W \cL(v)$, it holds
$(\gradM \cL(v) , w)_W = \cL'(v)(w) = (\nabla_W \cL(v) , w)_W$ for all $w \in \cT_v$, which  yields $\gradM \cL(v) = P^W_{\cT_v} \nabla_W \cL(\v)$.
\end{proof}

\begin{remark}
When the space $W$ at $v$ is equipped with the inner product induced by the Hessian  $H_\cL(v)$, NGD will be referred to as a Gauss-Newton NGD.
\end{remark}

\begin{remark}
\label{rmk:one-step-convergence}
    If the functional loss $\cL$ is defined by the quadratic elliptic functional $\cL(v) = \frac 1 2 H_\cL(v, v) - \ell(v)$, with $\ell$ a linear form, then $\cL'(v)(w)=H_\cL(v, w) - \ell(w) = H_\cL(v-u, w)$, with $u$ the minimizer of $\cL$. With the tangent space equipped with the inner product induced by $H_\cL(\vk)$, the direction $p^{(k)} \in \cT_\vk$ is defined by $H_\cL(\psikT p^{(k)} , w) =  - \cL'(\vk)(w)  = H_\cL(u-\vk, w)  $ for all $w\in \cT_k$. This means that $\psikT p^{(k)} = P^W_{\cT_\vk} (u-\vk)$ is the orthogonal projection onto $\cT_k$ of the error $u-\vk$ with respect to the inner product induced by $H_\cL$.
    If $\cM$ is a linear space,  the convergence towards the optimizer $P^W_{\cM} u$ is  realized in one step with a step size $s=1$.
\end{remark}

\begin{remark}
Interestingly, in the case of a Hilbert space $V$, when the tangent space metric $W$ is given by the Hessian of the loss $H_\cL(v)$ at $v$, identified with an operator $\mathcal{H}_\cL(v) : V \to V$, we have
\begin{align*}
&({\gradM \cL(v),w})_{W}  = H_\cL(v)({\gradM \cL(v), w})  = \cL'(\v)(w) \quad \forall w\in \cT_v \\
\iff& ({\mathcal{H}_\cL(v)) \gradM \cL(v) , w}_V = ({\nabla \cL(\v), w})_V \quad \forall w\in \cT_v.
\end{align*}
The direction $p^{(k)} = - \gradM \cL(\vk)$ thus obtained corresponds to the orthogonal projection of the Newton correction $\mathcal{H}_\cL(\vk)^{-1} \nabla  \cL(\vk) $ with respect to the $W$-inner product induced by $\mathcal{H}_\cL(\vk)$, that is  $p^{(k)}=- P_{\cT_k}^W \mathcal{H}_\cL(\vk)^{-1} \nabla  \cL(\vk) $.
In the context of PDEs this corresponds to the classical variational setting with $P^W_{\cT_k}$ being a Galerkin projection onto the approximation space $\cT_k$.
\end{remark}

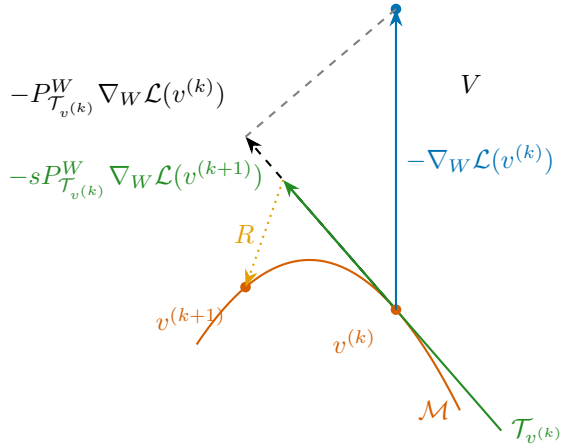
\begin{figure}[h]
 \centering
\begin{tikzpicture}[>=Stealth, line width=0.8pt,
    declare function={
        manifold(\x) = -0.5*(\x-2)*(\x+2);
    }
    ]
    \pgfmathsetmacro{\px}{1.15};
    \pgfmathsetmacro{\py}{manifold(\px)};
    \pgfmathsetmacro{\slope}{-\px };

    \draw[red] plot[domain=-1.5:2, samples=50] (\x, {manifold(\x)}) node[left] {$\color{red} \cM$};

    \pgfmathsetmacro{\w}{1.4}
    \pgfmathsetmacro{\pr}{1.0}
    \draw[ForestGreen, thick] (\px-\w,\py-\slope*\w) -- ({\px+\w}, {\py+\slope*\w}) node[right] {$\cT_{v^{(k)}}$};
    \draw[->, black, dashed, thick] (\px,\py) -- ({\px-2*\pr}, {\py-\slope*2*\pr}) node(projection) [label=above left:{$-P^W_{\mathcal{T}_{v^{(k)}}}\nabla_W  \mathcal{L}(v^{(k)})$}] {};

    \draw[->, ForestGreen, thick] (\px,\py) -- ({\px-1.5*\pr}, {\py-1.5*\slope*\pr}) node(update) [label=left:{${\color{ForestGreen}-s P^W_{\cT_{v^{(k)}}} \nabla_W \cL(v^{(k+1)})}$}] {};

    \fill[blue] (\px,{manifold(\px)+4}) circle (2pt) node(u) [label=above right:{}] {};
    \node[black] at (\px+1,{manifold(\px)+3}) {$V$};
    \fill[red] (\px,{manifold(\px)}) circle (2pt) node(v) [label=below left:{${\color{red} v^{(k)}}$}] {};
    \fill[red] (\px-2,{manifold(\px-2)}) circle (2pt) node(vp) [label=below left:{${\color{red} v^{(k+1)}}$}] {};

    \draw[->, thick, blue] (v.center) -- (u.center) node[midway, right, black] {${\color{blue} -\nabla_W \cL(v^{(k)})}$};
    \draw[thick, gray, dashed] (u.center) -- (projection.center);
    \draw[->, thick, dotted, orange] (update.center) to[bend right=0] node[midway, orange, left] {$R$} (vp.center);

\end{tikzpicture}

    \caption{Diagram of a natural gradient step for $W\subset V$.
    The retraction $R$ takes the proposed update ${\colorT -sP^W_{\cT_k} \nabla_W \cL(\vk)}=s\psikT \pk$ in the tangent space $\cT_{\vk}$ and yields the next iterate ${\colorM \vkp}=D(\paramkp)=D(\paramk + s\pk)$.}
    \label{fig:nat-grad}
\end{figure}

\paragraph{Functional gradient flow.}
The natural gradient descent appears as a time discretization of a gradient flow in function space
\begin{equation}
\frac{\partial v}{\partial t} = - \gradM \cL(v(t)) , \label{eq:natural-gradient-flow}
\end{equation}
where $\gradM \cL(v(t))$ is the Riesz representer of $\cL'(v(t))$ in the tangent space $\cT_{v(t)} $ to $\cM$ at $v(t).$  This defines a dynamical system in the manifold.
Given the parametrization $v(t) = D(\theta(t))$, the functional moment $\frac{\partial v}{\partial t}$ is related to the parameter moment $\frac{\partial \theta}{\partial t}$ through the relation
$$
\frac{\partial v}{\partial t} = \frac{\partial D}{\partial \theta}(\theta(t)) \frac{\partial \theta}{\partial t} = \psi(\theta(t))^T p(t) \in \cT_{v(t)}.$$
If $\dim \cT_{v(t)} = d$, i.e.  $\psi(\theta(t))$ forms a basis of $\cT_{v(t)}$, there is a one-to-one relation between  $p(t)$ and  $\frac{\partial v}{\partial t}$, and the functional gradient flow defines a dynamical system in parameter space.
Conversely, if $\dim \cT_{v(t)} < d$, i.e. $\psi(\theta(t))$ is a redundant generating system of $\cT_{v(t)}$, the functional gradient flow remains (locally) well defined, even if there   may not exist a well defined dynamical system in the parameter space.

Assuming that the manifold is embedded in the  Hilbert space $V$ and we choose $W=V$, or in the case where $W \subset V$, we deduce from Lemma \ref{lem:projection}  the
following interpretation of the natural gradient flow as a dynamical system induced by a Dirac-Frenkel principle.
\begin{proposition}
\label{prop:projection}
Assuming $W\subset V$,  the natural gradient dynamics is
$$\dvdt = -P^W_{\cT_{v(t)}} \nabla_W \cL(\v(t))$$
where $P^W_{\cT_{v(t)}}$ is the $W$-orthogonal projection onto $\cT_{v(t)}$. In particular, for $V$ a Hilbert space and $W=V$,
$$\dvdt = -P^V_{\cT_{v(t)}} \nabla \cL(\v(t)).$$
\end{proposition}

We will now define new algorithms in function space and identify corresponding algorithms in parameter space.

\section{Natural momentum algorithms}
\label{sec:natural-momentum}

What can be gained by adding momentum to NGD? A first observation is that the dynamics of a gradient flow even in function space would get stuck in the first \textbf{local minima} it encounters due to the non linear nature of the approximation manifold $\cM$ or the loss function $\cL$ thus adding inertia into the dynamics could allow further landscape exploration.

Another argument stems from the error committed by discretizing the natural gradient flow dynamics. Indeed, the direction given by the natural gradient is the best from a functional perspective but this is only valid in a small vicinity of the current iterate. Taking a discrete step by updating the parameters through \eqref{eq:ngd-update} follows a straight line in parameter space which deviates from the continuous functional dynamics. This occurs because it does not consider the following facts.
\begin{itemize}
    \item The \textbf{manifold $\cM$ is non linear} (e.g. a neural network) meaning that $G(\theta)$ will vary with $\theta$ thus the straight line in parameter space, although being a good direction locally, will deviate from the optimal path. This deviation occurs because the model hessian $H_D$ is only zero for linear manifolds.
    \item  The \textbf{loss function $\cL$ is not a quadratic form} (e.g., cross entropy) thus even if we precondition  by the hessian $H_\cL$  of the functional loss, the functional update direction does not point directly to the minimizer contrary to \cref{rmk:one-step-convergence}.
\end{itemize}

Furthermore, the direction $p^{(k)}=-\Ginvk\nabla L(\paramk)$ is typically never obtained exactly as it involves an estimation of $G^{(k)}$ and $\LossGrad(\paramk)$, typically through Monte Carlo estimation. In addition,  the pseudo-inverse of $G^{(k)}$ is usually replaced  by  a \textbf{regularized version}, see \Cref{sec:nat-grad}.

We argue that adding information from previous iterations can help to correct these deviations resulting in an acceleration of the convergence of NGD.

\subsection{Classical momentum strategies}
\label{sec:classical-momentum}

\subsubsection{Heavy-Ball}
The \textit{heavy ball method} in parameter space was proposed originally in \cite{polyakMethodsSpeedingConvergence1964} as the discretization of the second order ordinary differential equation
\begin{equation}
\frac{\partial^2 \theta}{\partial t^2} + b(t) \frac{\partial \theta}{\partial t} = -\LossGrad(\theta)
\label{eq:hb-sos}
\end{equation}
or its equivalent first order coupled system
\begin{subequations}
    \begin{align}
\frac{\partial \theta}{\partial t} &= p(t) \label{eq:hb-vupdate}\\
\frac{\partial p}{\partial t} &= - b(t) p(t) - \LossGrad(\theta(t)), \label{eq:hb-fos}
\end{align}\label{eq:hb-param}
\end{subequations}
where we have introduced the momentum (or velocity) variable $p$.

The insight behind is that $\LossGrad$ no more drives the velocity but the acceleration, so that flat regions of the loss landscape can be quickly escaped from  or local minima avoided as the inertia can pull the optimization out of it. The addition of the damping term with coefficient $b(t)$ helps to avoid the oscillations of the otherwise conservative dynamics.

Discretizing \eqref{eq:hb-param}  with a step size $h_k$ (following \cite{attouchRavineMethodNesterov2022,attouchFirstorderOptimizationAlgorithms2020}) yields
$$
\frac{\pk-\pkm}{h_k} = - b_k \pkm -\LossGrad(\paramk)
$$
and
$$
\frac{\paramkp-\paramk}{h_k} = \pk = (1- b_kh_k) \pkm -h_k\LossGrad(\paramk)
$$
which finally gives, by rearranging terms,
\begin{subequations}
\begin{align}
\paramkp &= \paramk  + h_k\beta_k \pkm -\alpha_k \LossGrad(\paramk) \label{eq:hb-momentum-version}\\
 &= \paramk  + \beta_k \frac{h_k}{h_{k-1}}(\paramk - \paramkm) -\alpha_k\LossGrad(\paramk) \label{eq:hb-param-version}
\end{align}\label{eq:hb}
\end{subequations}
with $\alpha_k=h_k^2=s_k$ and $ \beta_k=(1-h_k b_k)$ and we have that previous gradients information  or previous parameter values appear in the update formulation through $\paramkm$  and $\pkm$.

\subsubsection{Nesterov}
Instead, the \textit{Nesterov} acceleration, originally proposed in \cite{nesterovMethodSolvingConvex1983}, required more than thirty years to be linked with a  dynamical system as \eqref{eq:hb-sos} using a  damping parameter depending on time as $b(t)=\frac 3 t$ \cite{suDifferentialEquationModeling2015}. The resulting discretization can be written in the following three steps iteration
\begin{subequations}
\begin{align}
    \yk &= \paramk + \beta_k h_{k-1} \pkm
    \label{eq:Nesterov-y} \\
    \paramkp &= \yk-\alpha_k\LossGrad(\yk) \label{eq:Nesterov-theta}\\
    \pk &= (\paramk - \paramkm)/h_k,
\end{align}
\end{subequations}
where $\beta_k=\frac{k-1}{k+1}\approx 1-\frac 3 k$ is related to the choice of $b(t)=\frac 3 t$ from \cite{suDifferentialEquationModeling2015}. We draw particular attention to the fact that the evaluation of the gradient is not performed at $\paramk$ as in \eqref{eq:hb-param-version} but instead at the intermediate position $\yk=\paramk+\beta_k(\paramk - \paramkm)$. The equation \eqref{eq:Nesterov-y} provides  the drift caused by the the momentum while the equation \eqref{eq:Nesterov-theta} is a typical gradient update but taken at the shifted position.

To design computationally efficient functional versions of Nesterov method it will be useful to work with the combined equation
\begin{equation}
    \paramkp = \paramk + \beta_k(\paramk - \paramkm)-\alpha_k\LossGrad(\paramk + \beta_k(\paramk - \paramkm)) \label{eq:Nesterov-combined}.
\end{equation}

\subsection{Natural Heavy-Ball algorithms}
\label{sec:heavy-ball}

\subsubsection{A natural version of Heavy-Ball}
\label{sec:nhb}

One approach to derive functional versions of  accelerated algorithms is to write
the dynamical system \eqref{eq:hb} in function space. A functional version of  \eqref{eq:hb} is
\begin{subequations}
\begin{align}
    \dvdt &= \cP(t) \\
    \dPdt &= -b(t)\cP(t) - \gradM \cL(v(t)) + \cN(t)  \label{eq:momentum-functional-full}
\end{align}\label{eq:functional-V}
\end{subequations}
with $\cP(t) \in V$ the functional momentum, $\gradM \cL(v(t))$ the manifold gradient such that $(\gradM\cL(v), w)_W = \cL'(v(t))(w)$ for all $w \in \cT_{v(t)}$,  and $\cN(t)$ in a complement of $\cT_{v(t)}$ such that the trajectory $v(t)$  remains in $\cM$.
The trajectory $v(t)$ remaining in $\cM$, the momentum $\cP(t) \in \cT_{v(t)}$ can be written $\cP(t)=\psi(t)^T p(t)$.
To define the evolution of $(v(t), \cP(t))$, it is sufficient to only consider a projection of the equation
\eqref{eq:momentum-functional-full} onto  $\cT_{v(t)}$ with respect to an inner product $(\cdot, \cdot)_X$, possibly different from the inner product $(\cdot, \cdot)_W$ defining the  gradient of $\cL$, and assume that $\cN(t)$ is $X$-orthogonal to $\cT_{v(t)}.$ This finally yields the dynamical system
\begin{subequations}
\begin{align}
    \dvdt &= \cP(t) \label{eq:functional-cP}\\
    P^X_{\cT_{v(t)}}\dPdt &= -b(t)\cP(t) - \gradM \cL(v(t)) \label{eq:momentum-continuus-equation}
\end{align}\label{eq:hb-continuus-equation}
\end{subequations}
where the second equation is equivalent to
$$
(\psi(t), \dPdt)_X = - (\psi(t) , b(t)\cP(t) - \gradM \cL(v(t)))_X.
$$
The choice of different spaces $W$ and $X$ may be relevant for  problems where the momentum and its time derivative do not belong to the same function spaces.
This becomes particularly relevant when $W$ stems from the Hessian of the loss $H_\cL$
whose role is a preconditioning of $\cL'$.

Similarly to what is done for classical momentum methods, we proceed to discretize the continuous equation \eqref{eq:momentum-continuus-equation} as
$$
(\psik, \frac{\Pk-\Pkm}{h_k})_X = -(\psik, b_k \Pkm + \gk)_X,
$$
with $\gk=\gradM\cL(\vk)$.
After rearranging the terms we obtain the following equation for $\Pk$
\begin{equation}
(\psik, \Pk)_X = \beta_k(\psik, \Pkm)_X - h_k (\psik, \gk)_X, \label{eq:momentum-tangent-space}
\end{equation}
with $\beta_k = 1-h_k b_k.$
Recalling that $\Pk=\psikT\pk$, $\Pkm=\psikmT\pkm$ and $\gk=\psikT\Ginvk\LossGrad(\paramk)$ we get
$$
(\psik, \psikT)_X\pk = \beta_k (\psik, \psikmT)_X\pkm - h_k (\psik, \psikT)_X \Ginvk \LossGrad(\paramk),
$$
or more concisely
\begin{equation}
\GXk\pk = \beta_k\GXkkm\pkm - h_k\GXk \Ginvk\LossGrad(\paramk), \label{eq:nhb-momentum-update}
\end{equation}
with $G_X^{(k)} \coloneqq ({\psik, \psikT})_X$ the Gram matrix of the generating system $\psi_k$ of $\cTk$ under the metric $X$ and $\GXkkm\coloneqq ({\psik, \psikmT})_X$ the cross Gram matrix between the generating systems of tangent spaces $\cTkm$ and  $\cTk$. We obtain  the functional momentum direction at iteration $k$
\begin{equation}
\Pk=\psikT\pk = \beta_k\psikT\GXinvk\GXkkm\pkm - h_k\psikT\Ginvk\LossGrad(\paramk).
\end{equation}

Finally, going back to \eqref{eq:functional-cP} we define
$$
v^{(k+1)} = R(\vk + h_k \cP^{(k)})
$$
that is
$$
v^{(k+1)} = D(\theta^{(k+1)}), \quad \theta^{(k+1)} = \theta^{(k)} + h_k \pk.
$$
This results in the formulation of a natural heavy ball (NHB) update rule
\begin{equation}
\paramkp
= \paramk + \frac{h_k}{h_{k-1}}\beta_k\GXinvk\GXkkm(\paramk - \paramkm) - \alpha_k\Ginvk\LossGrad(\paramk) \label{eq:nhb-update}
\end{equation}
where we used the identity $\pkm=\frac{\paramk - \paramkm}{h_{k-1}}$. Here $\beta_k=(1-h_k b_k)$ and $\alpha_k=h_k^2$ are the same as in the classical algorithm.

Note that the term
$$\psikT\GXinvk\GXkkm\pkm = \psikT\GXinvk ({\psik, \psikmT})_X \pkm = P_{\cT_{\paramk}}^X \Pkm$$
corresponds to the $X$-orthogonal projection of the previous momentum $\Pkm=\psikmT\pkm\in \cT_{\paramkm}$ into $\cT_{\paramk}$.

In summary, we have the following definition.
\begin{definition}[Natural Heavy-Ball (NHB)]
The natural Heavy-Ball method in function space can be written as
\begin{align*}
    \vkp = R[\vk + \frac{h_k }{h_{k-1}}\beta_k P^X_{\cT_{\vk}} \psikmT (\paramk-\paramkm)  - \alpha_k \gradM \cL(\vk)].
\end{align*}
When we choose the retraction given by \eqref{eq:retraction} we obtain the parameter update rule
\begin{equation}
    \paramkp
= \paramk + \frac{h_k}{h_{k-1}}\beta_k\GXinvk\GXkkm(\paramk - \paramkm) - \alpha_k\Ginvk\LossGrad(\paramk).
\end{equation}
\end{definition}

\begin{remark}
When the space $W$ at $v$ is equipped with the inner product induced by the Hessian  $H_\cL(v)$, NHB (and subsequent variants) will be referred to as a Gauss-Newton NHB.
\end{remark}


\subsubsection{Quasi-Natural Heavy-Ball}
\label{sec:qnhb}

The previous approach stems from a functional perspective. However, the cost of calculating the cross-Gram matrix and the potential for estimation errors in low sample regimes may limit its advantages over a basic NGD update. In this context one could wonder if combining an NGD update with a classical parametric momentum update could  yield similar results without extra computational cost. Instead of the NHB momentum update given by \eqref{eq:nhb-momentum-update} we could simply have
\begin{equation}
    \pk = \beta_k \pkm - h_k \Ginvk\nabla L(\paramk) \label{eq:qnhb-momentum-update}
\end{equation}
leading to the following definition.
\begin{definition}[Quasi-Natural Heavy-Ball (QNHB)]
The quasi-natural Heavy-Ball method in function space can be written as
\begin{align*}
    \vkp = R[\vk + \frac{h_k }{h_{k-1}}\beta_k P^X_{\cT_{\vk}} \psikT (\paramk-\paramkm)  - \alpha_k \gradM \cL(\vk)].
\end{align*}
When we choose the retraction given by \eqref{eq:retraction} we obtain the parameter update rule
\begin{equation}
    \paramkp
= \paramk + \frac{h_k}{h_{k-1}}\beta_k(\paramk - \paramkm) - \alpha_k\Ginvk\LossGrad(\paramk).
\end{equation}
\end{definition}
From a functional perspective what we have done is to use $\psik$ instead of $\psikm$ in the definition of the momentum term replacing $\Pkm = \psikmT (\paramk-\paramkm) h_{k-1}^{-1}$ by $\Pkm = \psikT (\paramk-\paramkm)h_{k-1}^{-1}$, that is, interpreting the momentum $\Pkm $ as an element of the tangent space at $\vk = D(\paramk)$. This translates into an update rule where instead of $\GXinvk\GXkkm\pkm$ we have just $\pkm$. This supposes that $\GXinvk\GXkkm$ is close to identity which is reasonable in certain conditions.
This approach is closely related with recent momentum enhanced NGD methods for solving PDEs: the Kaczmarz inspired SPRING method \cite{Kaczmarz2024} and the Woodbury Energy Natural Gradient (ENGD-W) \cite{ImprovingEnergyNatural2025}.

We define the maximal curvature of the manifold $\cM$ at point $v=D(\theta)$ as
$$
\kappa_{max}(\theta) = \underset{q\in \R^d;\\ \|q\|_2=1}{\max} \|H_D(\theta)(q, q)\|_X
$$
recalling that $H_D(\theta)(\cdot, \cdot) : \mathbb{R}^d \times \mathbb{R}^d \to V$ is the parametric Hessian introduced in \Cref{sec:gauss-newton}.
\begin{proposition}
\label{prop:qnhb}
It holds
$$
 \|\psikT \overline{\pk} -\psikT \pk \|_X\leq h_{k-1} \beta_k \kappa_{max}(\paramk) \|\pkm\|_2^2 +  o(h_{k-1} \|\pkm\|_2^2)
$$
with $\overline{\pk}$ and  $\pk$ respectively given by \eqref{eq:qnhb-momentum-update} and \eqref{eq:nhb-momentum-update}.
\end{proposition}
We leave the proof to \Cref{sec:proof-qnhb}. From this we can conclude that NHB and QNHB methods follow a similar dynamics in flat or nearly-flat manifolds $\cM$ as $\kappa_{max}(\paramk) \approx 0$.

\subsubsection{Natural Heavy-Ball with functional difference}
\label{sec:nhb-fd}

Another way to avoid calculating the cross-Gram matrix in the NHB method is to approximate the momentum $\Pkm=\psikmT \pkm$ by $(\vk-\vkm)/h_{k-1}$ on the tangent space $\cT_k$, that is
$$
({\psik, \Pkm})_X \approx \frac{1}{h_{k-1}}({\psik, \vk-\vkm})_X.
$$
Replacing in \eqref{eq:momentum-tangent-space} we obtain the following momentum update

\begin{equation}
({\psik, \Pk})_X = \frac{\beta_k}{h_{k-1}} ({\psik, \vk-\vkm})_X - h_k ({\psik, \gk})_X, \label{eq:momentum-tangent-space-fd}
\end{equation}
or in terms of parameters, recalling that $\Pk=\psikT\pk$ and $\gk=\psikT\Ginvk\LossGrad(\paramk)$,
\begin{equation}
\GXk\pk = \frac{\beta_k}{h_{k-1}}z^{(k-1)} - h_k\GXk \Ginvk\LossGrad(\paramk). \label{eq:nhbfd-momentum-eq}
\end{equation}
with $z^{(k-1)}\coloneqq ({\psik, \vk-\vkm})_X.$
Discretizing \eqref{eq:functional-cP} as $v^{(k+1)} = D(\theta^{(k+1)})$ with $ \paramkp = \paramk + h_k \pk$,
we obtain the following natural Heavy-Ball functional difference (NHB-FD) momentum update rule
\begin{equation}
\pk = \frac{\beta_k}{h_{k-1}}\GXinvk z^{(k-1)} - h_k\Ginvk\LossGrad(\paramk) \label{eq:nhbfd-momentum-update}
\end{equation}
which leads to the following definition. 
\begin{definition}[Natural Heavy-Ball with Functional Difference (NHB-FD)]
The natural Heavy-Ball with functional difference method can be written in function space as
\begin{align*}
    \vkp = R[\vk + \frac{h_k }{h_{k-1}}\beta_k P^X_{\cT_{\vk}} (\vk-\vkm)  - \alpha_k \gradM \cL(\vk)],
\end{align*}
where the parameters  $\beta_k=1-h_k b_k$ and $\alpha_k=h_k^2=s_k$ are the same as in the classical algorithms.
When we choose the retraction given by \eqref{eq:retraction} we obtain the parameter update rule
\begin{equation}
    \paramkp
= \paramk + \frac{h_k}{h_{k-1}}\beta_k\GXinvk z^{(k-1)} - \alpha_k\Ginvk\LossGrad(\paramk), \label{eq:nhbfd-update}
\end{equation}
\end{definition}

This alternative version of natural momentum has the advantage of avoiding the storage at iteration $k$ of evaluations of the gradient $\psikm$ (or new evaluations in an active learning setting) required to estimate the cross-Gram matrix $\GXkkm$. Instead only two forward passes to the model $D(\paramk)$ and $D(\paramkm)$ are required to obtain the difference $\vk-\vkm$.

Finally, we can also quantify the difference between NHB and NHB-FD dynamics under similar assumptions as in \Cref{prop:qnhb}.
\begin{proposition}
\label{prop:qnhb-bis}
It holds
$$
\|\psikT \widehat{\pk}-\psikT \pk\|_X\leq \frac{h_{k-1}}{2} \beta_k \kappa_{max}(\paramkm) \|\pkm\|_2^2 { + o(h_{k-1}\|\pkm\|_2^2)}
$$
where   $\widehat{\pk}$ and  $\pk$ are given by \eqref{eq:nhbfd-momentum-update} and \eqref{eq:nhb-momentum-update} respectively.
\end{proposition}
We leave the proof to \Cref{sec:proof-nhb-fd}. We note that the moments approximations introduced by NHB-FD and QNHB are of the same order.

\subsection{Natural Nesterov}
\label{sec:nesterov}

A functional version of the Nesterov method in the case of a linear space $\cM$ constructs two sequences $(\wk = D(\yk))$ and $(\vk = D(\paramk))$ in $\cM$ defined by
\begin{align*}
&\wk = \vk + \beta_k h_{k-1} \cP^{(k-1)} \\
&\vkp = \wk - \alpha_k \gradM \cL(\wk),
\end{align*}
with the momentum term $\cP^{(k-1)} = (\vk - \vkm) / {h_{k-1}}.$
In the case of a nonlinear manifold $\cM$, a natural Nesterov algorithm can be defined as
\begin{align*}
&\wk = R(\vk + \beta_k h_{k-1} P^X_{\cT_{\vk}} \cP^{(k-1)}) \\
&\vkp = R(\wk - \alpha_k \gradM \cL(\wk)),
\end{align*}
which involves a projection of the momentum  $\cP^{(k-1)}$  onto the tangent space $\cT_{\vk}$ and two retractions. Note that the gradient $\gradM\cL(\wk)$ is an element of $\cT_{\wk},$ which is the Riesz representer of $\cL'(\wk)$ in $\cT_{\wk}$.
This yields two variants depending on the interpretation of the momentum $\cP^{(k-1)}$ for a nonlinear manifold, either $${h_{k-1}} \cP^{(k-1)} = \vk - \vkm   = D(\paramk) - D(\paramkm)  \in \cM - \cM,$$ or
$$
{h_{k-1}} \cP^{(k-1)} = \psi(\paramkm)^T (\paramk-\paramkm) \in \cT_{\vkm}.
$$
Letting $d_k$ be such that $h_{k-1}P^X_{\cT_{\vk}} \cP^{(k-1)} = \psi(\paramk)^T d_k$, these two choices  correspond to \begin{align}
d_k &= G_X(\paramk)^\dagger G_X(\paramk,\paramkm) (\paramk-\paramkm) \label{eq:dk-nesterov} \quad \text{or} \\
d_k &= G_X(\paramk)^\dagger (\psi(\paramk) , \vk-\vkm)_X, \label{eq:dk-nesterov-fd}
\end{align}
where $G_X(\paramk)$ is the Gram matrix of $\psi(\paramk)$ and $G_X(\paramk,\paramkm)$ is the cross-Gram matrix of $\psi(\paramk)$ and $\psi(\paramkm)$, both with respect to the $X$-inner product.  The second variant is interesting in an active learning setting, since it does not require new evaluations of previous gradients $\psi(\paramkm)$.
\begin{definition}[Natural Nesterov I (NN-I and NN-I-FD)]
\label{def:nn-I}
We define sequences $(\vk = D(\paramk))$ and $(\wk=D(\yk))$ in $\cM$ such that
\begin{align*}
&\wk = R(\vk + \beta_k h_{k-1}    P^X_{\cT_{\vk}}\cP^{(k-1)})    \\
&\vkp = R(\wk - \alpha_k \gradM \cL(\wk)).
\end{align*}
Letting $h_{k-1}P^X_{\cT_{\vk}}\cP^{(k-1)} = \psi(\paramk)^T d_k$ with $d_k$ as in \eqref{eq:dk-nesterov} or \eqref{eq:dk-nesterov-fd}, and choosing
the retraction \eqref{eq:retraction}, it yields the update rule
\begin{align*}
\yk &= \paramk + \beta_k d_k
\\
\paramkp &= \yk - \alpha_k G(\yk)^\dagger \nabla L(\yk).
\end{align*}
\end{definition}
The above algorithms involve retractions at two different points $\vk$ and $\wk$ in the manifold, which requires the computation of pseudo-inverses $G_X(\paramk)^\dagger$ and $G(\yk)^\dagger$. In the case where $X$ and $W$ metrics are the same, we would like to avoid computing two such matrices. For that, we will only consider retractions at point $\vk$, which requires replacing $\gradM \cL(\wk) \in \cT_{\wk}$ by the Riesz representer of $\cL'(\wk)$ in $\cT_{\vk}$, which is  $\mathcal{R}_{\cT_{\vk}} \cL'(\wk)_{\mid \cT_{\vk}} = \psi(\paramk)^T G(\paramk)^\dagger \cL'(\wk)(\psi(\paramk))$. This is feasible when $W \subset V$.
This yields the following definition. 
\begin{definition}[Natural Nesterov II (NN-II and NN-II-FD)]
We define sequences $(\vk = D(\paramk))$ and $(\wk=D(\yk))$ in $\cM$ such that
\begin{align*}
&\Delta_k = \beta_k  h_{k-1}    P^X_{\cT_{\vk}}\cP^{(k-1)} \\
&\wk = R(\vk + \Delta_k )    \\
&\vkp = R(\vk + \Delta_k - \alpha_k \mathcal{R}_{\cT_{\vk}} \cL'(\wk)).
\end{align*}
Letting $h_{k-1}P^X_{\cT_{\vk}}\cP^{(k-1)} = \psi(\paramk)^T d_k$ with $d_k$ as in \eqref{eq:dk-nesterov} or \eqref{eq:dk-nesterov-fd}, and choosing
the retraction \eqref{eq:retraction}, it yields the update rule
\begin{align*}
\yk &= \paramk + \beta_k d_k
\\
\paramkp &= \yk - \alpha_k G(\paramk)^\dagger \cL'(\wk)(\psi(\paramk)).
\end{align*}
\end{definition}

Let us emphasize that all variants introduced boil down to classical Nesterov in the case where $\cM$ is a linear space.

\begin{remark}
\label{rmk:nesterov-tangent}
An alternative approach to avoid two retractions is to replace $\cL'(\wk)$ with $\cL'(\vk + \Delta_k)$. However, this requires evaluating $\cL'$ outside the manifold. Although this offers negligible benefit with the retraction in \eqref{eq:retraction}, it becomes advantageous when working with computationally expensive retractions.
\end{remark}

\begin{remark}
\label{rmk:quasi-nesterov}
 Another possibility is to define ${h_{k-1}} \cP^{(k-1)} = \psi(\paramk)^T (\paramk-\paramkm) \in \cT_{\vk}$ instead of ${h_{k-1}} \cP^{(k-1)} = \psi(\paramkm)^T (\paramk-\paramkm) \in \cT_{\vkm}$ leading to computationally efficient versions of natural Nesterov with features similar to QNHB. One obtains an $\yk$-update rule
 $$\yk=\paramk + \beta_k(\paramk-\paramkm)$$
 that does not require the computation of the cross-Gram matrix neither the inversion of the Gram matrix. However, in our numerical experiments, this alternative version had the tendency to converge  more slowly or even diverge unless one uses a smaller value for $\beta_k$ deviating from the Nesterov one.
\end{remark}

\begin{remark}
When the space $W$ at $v$ is equipped with the inner product induced by the Hessian  $H_\cL(v)$, NN algorithms  will be referred to as Gauss-Newton NN algorithms.
\end{remark}

\section{Numerical experiments}
\label{sec:numerics}

In this section we compare NGD with the proposed natural momentum methods. The first two examples are benchmarks
in which the use of NGD was already shown to be better than classical gradient descent with momentum \cite{parkAdaptiveNaturalGradient2000} so we limit the comparison to the former. The last two experiments show the advantage of using natural momentum strategies over NGD in the context of Physics Informed Learning where the task is to solve a partial differential equation (PDE) by minimizing its residual. In this case the advantage of using NGD over L-BFGS or classical momentum strategies has been shown in \cite{müllerAchievingHighAccuracy2023,schwenckeANaGRAMNaturalGradient2024,müllerPositionOptimizationSciML2024}.

The approximation class $\cM$ for the first two examples corresponds to the class of shallow neural network with $10$ neurons while for the last two we use a two hidden layer neural network with $5$ neurons in each layer. The experiments are repeated $10$ times for each optimization method varying the parameter initialization randomly following a standard normal distribution. For each problem the training is stopped when a given error threshold is reached or when the number of iterations exceeds a certain prescribed value.

In the first, third and fourth examples, after a possible reformulation, the problem is formulated as a least-squares approximation problem in $V = L^2_\mu$ and we use spaces $W=L^2_\mu$ and $ X = L^2_\mu$. In the second example, the loss function is the cross-entropy and we consider for $X$ the space $L^2_\mu$ and for $W$  either $L^2_\mu$ or the normed space induced by the Hessian $H_\cL$. In the latter case,  the NGD is referred to Gauss-Newton NGD (GN NGD), while our accelerated methods are referred to Gauss-Newton variants of Heavy-Ball or Nesterov.

\subsection{Practical considerations}

Choosing the learning rate and the momentum term for a given problem is a problem in itself. In \cite{parkAdaptiveNaturalGradient2000} the authors searched the values that gave the fastest convergence and highest success rate, however, they didn't specify the exact methodology to achieve those values. In \cite{liAcceleratedNaturalGradient2025} the authors compare their proposed strategies to existing ones by choosing the best one given a grid search between two to four values selected in advance with no clear explanation of why those and not others.

In our experiments, we found that using gradient norm clipping as in \cite{gradclipping2013} such that
$$\alpha_k=\min \left\{ \frac{1}{\| \gk \|}, 0.1 \right\},$$
with $\gk$ the direction in parameter space of the gradient term (e.g. for NGD or NHB one has $\gk=\Ginvk\LossGrad(\paramk)$),
gives convergence dynamics that are relatively fast with none or just occasional overshoots while not requiring any specific method for searching a good learning rate. The drawback is that at the beginning the dynamics is slower than with optimized learning rates, which gives another reason to use momentum methods to accelerate it.
\newline

The inertial term for all heavy-ball methods was fixed at $\beta_k=0.5$ for all experiments. We observed that although sometimes higher values might give faster dynamics, it was sometimes at the expense of instabilities or oscillations.
\newline

The momentum terms presented here require the resolution of specific system of equations, mainly: the functional momentum update $\pk=\GXinvk \GXkkm \pkm=\GXinvk ({\psik, \psikmT \pkm})_X$ and the functional difference update $\pk=\GXinvk ({\psik, \vk-\vkm})_X$. As we have seen this was originally obtained from the system
$$\GXk\pk = ({\psik, \psikT})_X\pk = ({\psik, \Pkm})_X$$
with either $\Pkm=\psikmT \pkm$ or $\Pkm=\vk-\vkm$. This is equivalent to solve the following minimization problem
\begin{equation}
\pk = \argmin_{c\in\R^d} \|\Pkm-\psikT c\|^2_X \label{eq-lsqrt}
\end{equation}
which is nothing but the $X$-orthogonal projection of $\Pkm$ onto $\cT_\vk$. In practice one will use the empirical inner products, but, as long as the same quadrature method is used to estimate $\GXk$ and $({\psik, \Pkm})_X$ we can use an efficient least-squares solver for \eqref{eq-lsqrt}. Furthermore, when $W=X$ we can do the same for the gradient term. This was the setting for our numerical experiments except for the classification problem when using the metric inherited by the Hessian of the loss $H_\cL$ for the Gauss-Newton methods.
\newline

To solve the systems of equations involving  Gram  matrices, we used as regularization method the \textit{spectral flooring}, see \Cref{sec:nat-grad}.


\subsection{Mackey-Glass}
Following \cite{parkAdaptiveNaturalGradient2000}  we first compare NGD with the momentum variants introduced in this paper on a regression task consisting on learning the Mackey-Glass chaotic time series generated from the equation
$$z(t+1)=(1-b)z(t)+a\frac{z(t-\tau)}{1+z(t-\tau)^{10}}$$
with $a=0.2$, $b=0.1$ and $\tau=17$. The model class $\cM$ consists of a shallow neural network with $10$ hidden units and the sigmoid activation function. The input space is $\R^4$ given by the present and past values of the dynamical system $x=(z(t), z(t-6), z(t-12), z(t-18))$. The task is then to predict the future value $y=z(t+6)$. We used $500$ data pairs $(x_i, y_i)_{i=1}^{500}$ taken in the interval of $t=[200, 700]$ for training and another $500$ on $t=[5000, 5500]$ for testing. The loss is the mean square error (MSE). We stopped when the MSE reached $2\times10^{-5}$. In \cref{fig:makey_convergence} we show the median per iteration of $50$ different realizations of the experiment.

The first observation is that the momentum methods reach the minimal MSE in less than half the number of iterations as NGD and also in less than half the time. We also see that the approximate natural Heavy-Ball methods NHB-FD and QNHB are far behind NHB in terms of iterations and time while still being faster than NGD by a wide margin. The NN-II (using $d_k$ as in \eqref{eq:dk-nesterov}) shows even faster acceleration than NHB, however, the convergence is slowed down by the oscillating behavior caused by the vanishing damping term (recall that $\beta_k \approx 1-3/k$ for Nesterov) as it does not arrive to compensate the gained velocity. The alternative method NN-II-FD, which is based on $d_k$ from \eqref{eq:dk-nesterov-fd}, shows a behavior similar to NN-II in the first iterations. However, after $20$ iterations it deviates presenting a slowed convergence curve similar to NGD in time and iteration's number. We leave the comparison with NN-I to \Cref{sec:appendix-nesterov}, which has a similar error-iterations trend as NN-I while being twice computationally intensive arriving to convergence in times comparable to NGD.

\begin{figure}[ht]
    \centering
    \includegraphics[width=0.45\linewidth]{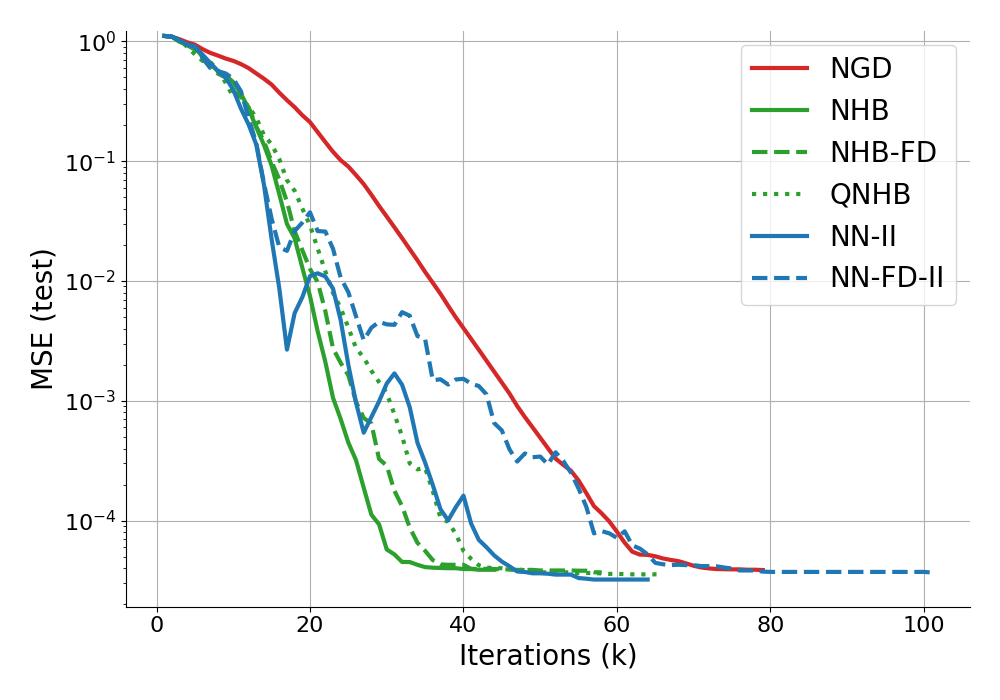}
    \includegraphics[width=0.45\linewidth]{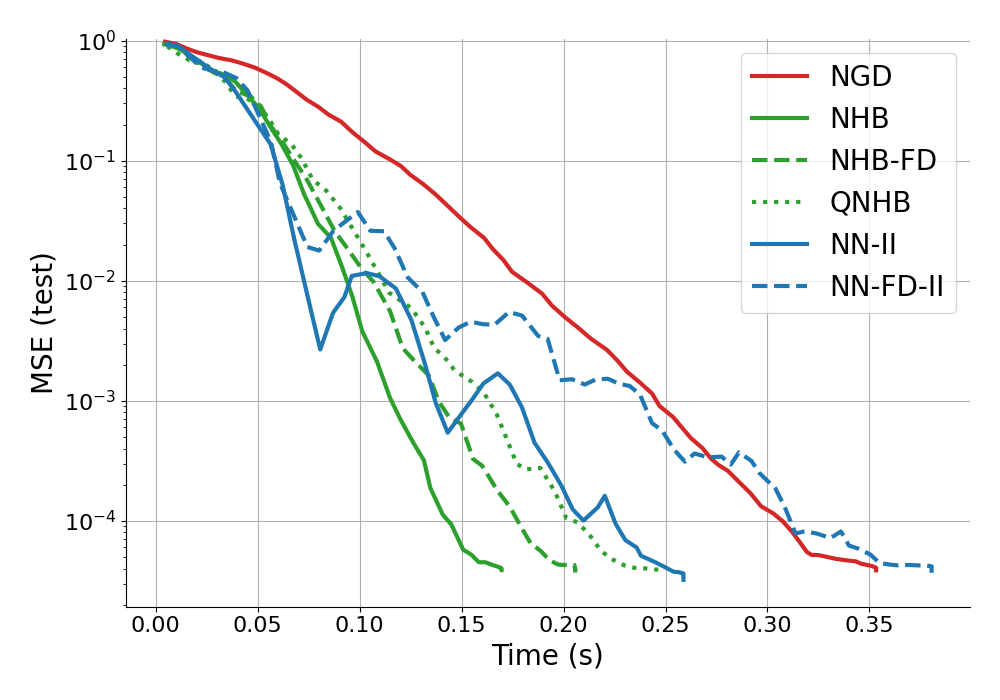}
    \caption{Convergence trends for the Mackey-Glass problem as a function of the number of iterations (left) or the time in seconds (right). In (red) the baseline method NGD. In (green) the NHB (solid line), NHB-FD (dashed line) and QNHB (dotted line). In (blue) the NN-II (solid line), NN-II-FD (dashed line).}
    \label{fig:makey_convergence}
\end{figure}

\subsection{Classification task}
Also following \cite{parkAdaptiveNaturalGradient2000} we compare the methods in the extended exclusive OR classification task which consists in separating the two classes arranged in a pattern as shown in \cref{fig:or-dataset}. Each of the nine clusters was generated following a standard Gaussian distribution. For training we pick $1800$ data points taking $200$ elements from each cluster so that classes are balanced. We pick $900$ for testing in a similar fashion. For this problem the learning process was stopped when the MSE was $3\times10^{-2}$ although we used for training the cross-entropy loss more appropriate for classification tasks. Here we also obtain an acceleration both in number of iterations (around half) and time (around $20\%$ less than NGD) when we use momentum methods. Line search methods reach faster the target loss in terms of iterations but take the double of time to converge.

In \Cref{fig:or-convergence}  we first note the inclusion of two baseline models: the Gauss-Newton NGD (GN NGD), which uses for $W$ the inner product induced by $H_\cL$ for obtaining the Riesz representer of $\cL'(\vk)$ in $\cT_\vk$, and NGD with $W$ begin the space $L^2_\mu$. We observe that taking into account $H_\cL$ effectively improves the learning process. In the same figure, we included the momentum methods that accelerate the NGD baseline with $X=W = L^2_\mu$. We see that all the natural heavy-ball versions behave similarly effectively accelerating the NGD dynamics. NN-II yields the fastest trend matching GN NGD at convergence in time and number of iterations. However, NN-II-FD shows a divergent pattern that could be avoided if we modify the Nesterov momentum constant, for example by setting $\beta_k \to \tilde \beta_k = \beta_k/2$, at the expense of slowing down the dynamics with respect to NN-II.

In \Cref{fig:gn-or-convergence} we compare Gauss-Newton NGD (GN NGD) with the Gauss-Newton variants of Heavy-Ball or Nesterov algorithms (GN NHB or GN NN) using also for $W$ the metric induced by the Hessian $H_\cL$.  We observe that all Gauss-Newton natural momentum algorithms effectively accelerate, though by a smaller margin, the learning dynamics of GN NGD, with GN NHB and GN NN-II having the steepest descent. Here we do not observe the blow-up of GN NN-II-FD.

\begin{figure}[ht]
    \centering
    \includegraphics[width=0.35\linewidth]{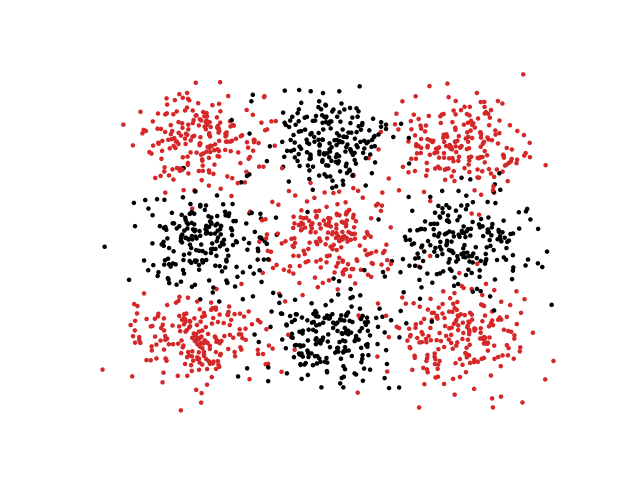}
    \caption{Extended exclusive OR dataset showing the two intertwined classes.}
    \label{fig:or-dataset}
\end{figure}

\begin{figure}[ht]
    \centering
    \includegraphics[width=0.45\linewidth]{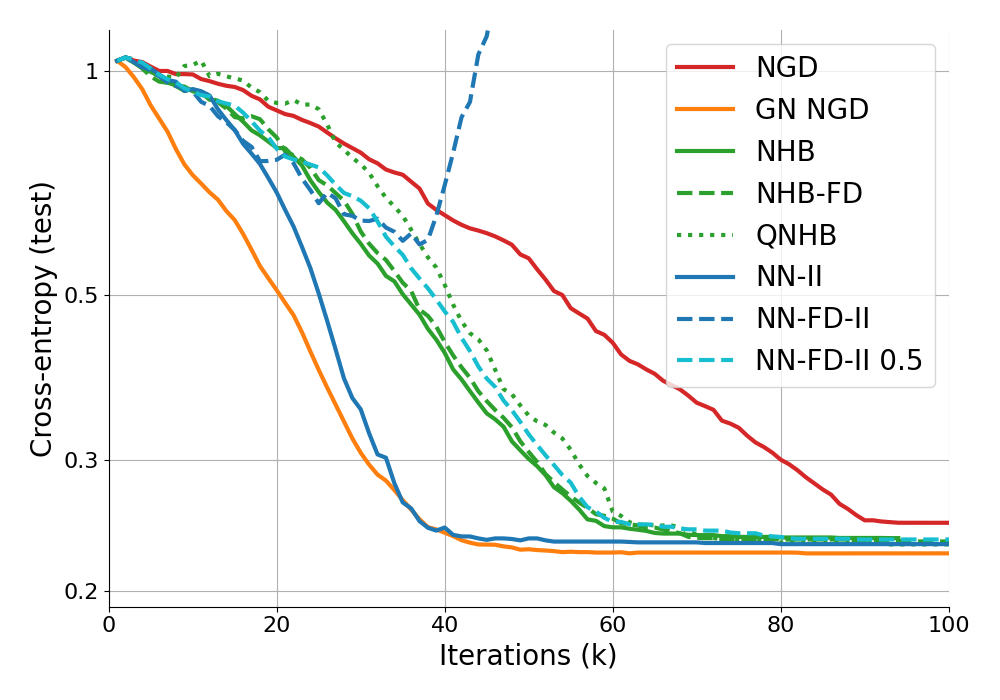}
    \includegraphics[width=0.45\linewidth]{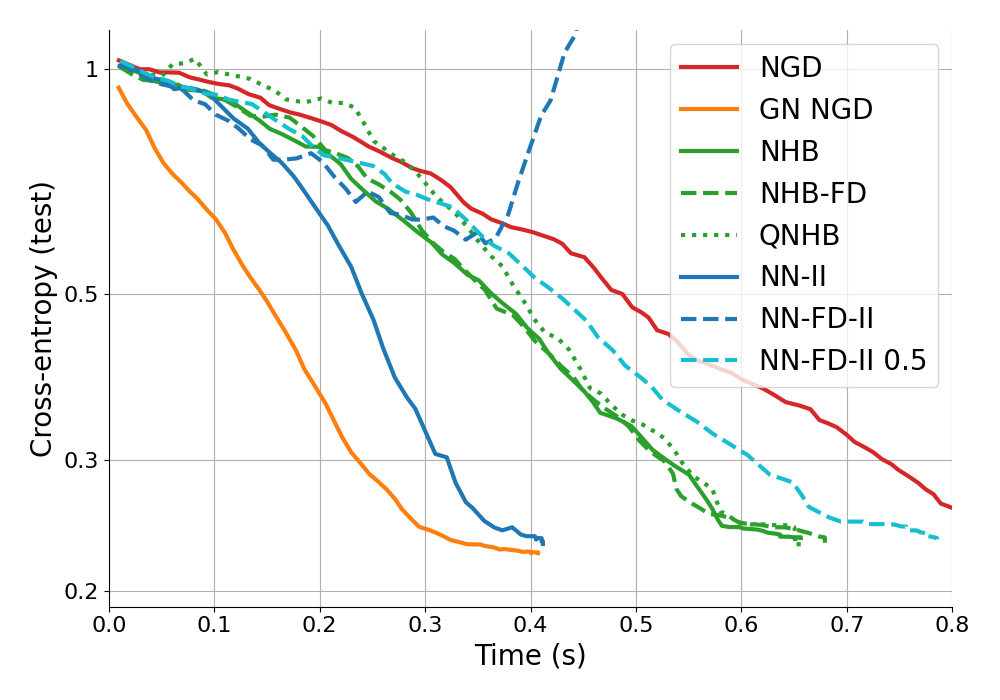}
        \caption{Convergence trends for the extended exclusive OR problem as a function of iterations (left) and time in seconds (right). The baseline methods are NGD (red) and Gauss-Newton NGD (GN-NGD) (orange). Here we show the natural momentum methods that use $W=X$. This methods can be seen as effectively accelerating the NGD baseline. In (green) the NHB (solid line), NHB-FD (dashed line) and QNHB (dotted line). In (blue) the NN-II (solid line), NN-II-FD (dashed line). In (cyan) NN-II-FD 0.5 (dashed line) which uses the modified Nesterov momentum constant $\tilde \beta_k = \beta_k/2$.}
    \label{fig:or-convergence}
\end{figure}

\begin{figure}[ht]
    \centering
    \includegraphics[width=0.45\linewidth]{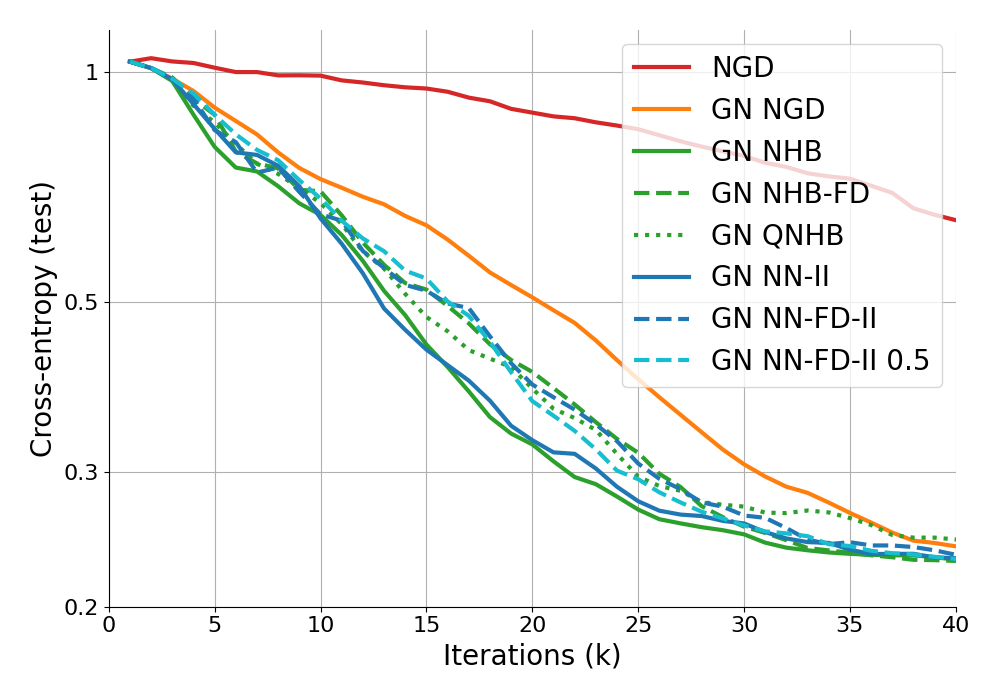}
    \includegraphics[width=0.45\linewidth]{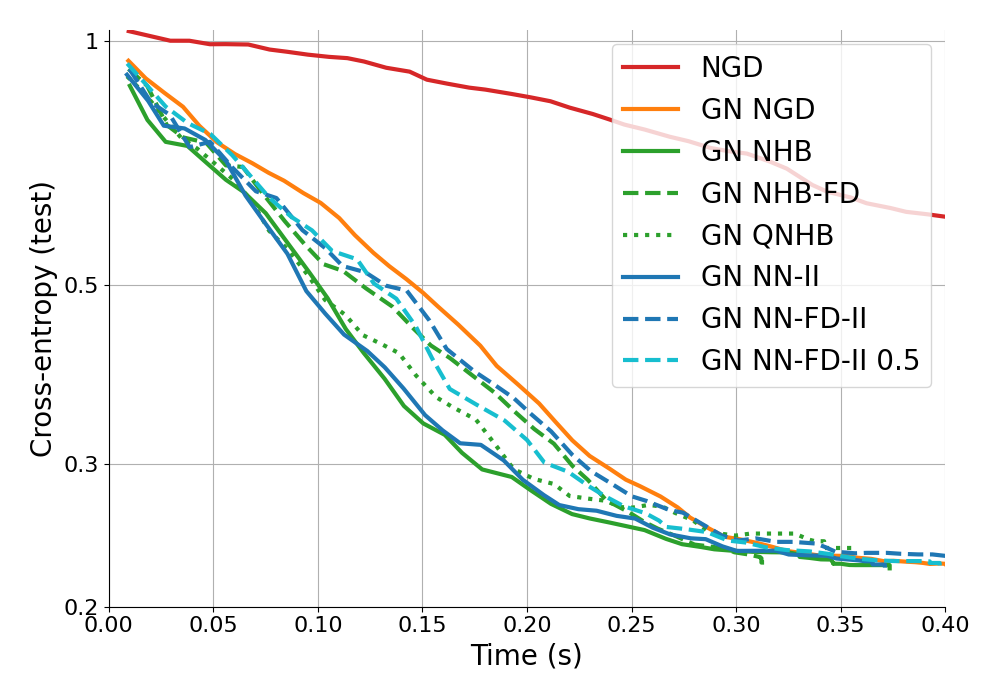}
        \caption{Convergence trends for the extended exclusive OR problem as a function of iterations (left) and time in seconds (right). The baseline methods are NGD (red) and Gauss-Newton NGD (GN NGD) (orange). Here we show the natural momentum methods that use the Hessian of the loss $H_\cL$ as the metric for obtaining the gradient Riesz representer. This methods can be seen as effectively accelerating the GN NGD baseline. In (green) the GN NHB (solid line), GN NHB-FD (dashed line) and GN QNHB (dotted line). In (blue) the GN NN-II (solid line), GN NN-II-FD (dashed line). In (cyan) GN NN-II-FD 0.5 (dashed line) which uses the modified Nesterov momentum constant $\tilde \beta_k = \beta_k/2$.}
    \label{fig:gn-or-convergence}
\end{figure}


\subsection{Physics informed learning}

The following two examples consist on the problem of finding the solution to a partial differential equation (PDE) by minimizing the strong form of its residual. That is, given a PDE
$$A(\u)(x)=f(x), \quad  x\in \Omega,$$
with boundary conditions, we want to find the element of the approximation manifold $\cM$ such that
\begin{equation}
{\colorM{v^*}} \in \argmin_{\v\in \cM} \|Au-f\|^2_{L^2(\Omega)}.
\label{eq:residual}
\end{equation}
Here the boundary conditions are satisfied by elements in $\cM$ (by modifying the output of the neural network accordingly). We can then consider instead of  a neural network as our model class $\cM$ the augmented class
$$\widetilde \cM = \{Av: v\in \cM\}.$$
In this case we have that the optimization problem is a minimization in $V = L^2(\Omega)$
$$A(v^*) \in \argmin_{w\in \widetilde \cM} \|\w-f\|^2_{L^2(\Omega)}.$$

\subsubsection{A linear advection-diffusion equation}

This example consists of a 1d linear advection diffusion equation with PDE operator given by
$$A(\u)(x)=-\mu\frac{\partial^2 \u}{\partial x^2}(x) + \frac{\partial \u}{\partial x}(x) = 1 =: f(x)$$
in the interval $\Omega=[0,1]$ with Dirichlet boundary conditions equal to zero at $x=0$ and $x=1$. The exact solution in this case is
$$\u(x)=-\frac{e^{\frac{1}{\mu}}}{e^{\frac{1}{\mu}}-1} + x + e^{\frac{1}{\mu}}-1$$
with $\mu=0.2$ in our experiments.

We use a two hidden layer neural network with $5$ neurons in each layer and hyperbolic tangent activation function. To enforce the boundary conditions we multiply the output of the neural network by $x(1-x)$. For training we use $1000$ equispaced collocation points in the domain $\Omega=[0,1]$ and $10000$ for testing.
In this example we also see that adding momentum helped to accelerate the convergence cutting by around half the number of iterations and time required to reach convergence as we see in \cref{fig:pinn_reacdiff_convergence}.

\begin{figure}[ht]
    \centering
    \includegraphics[width=0.45\linewidth]{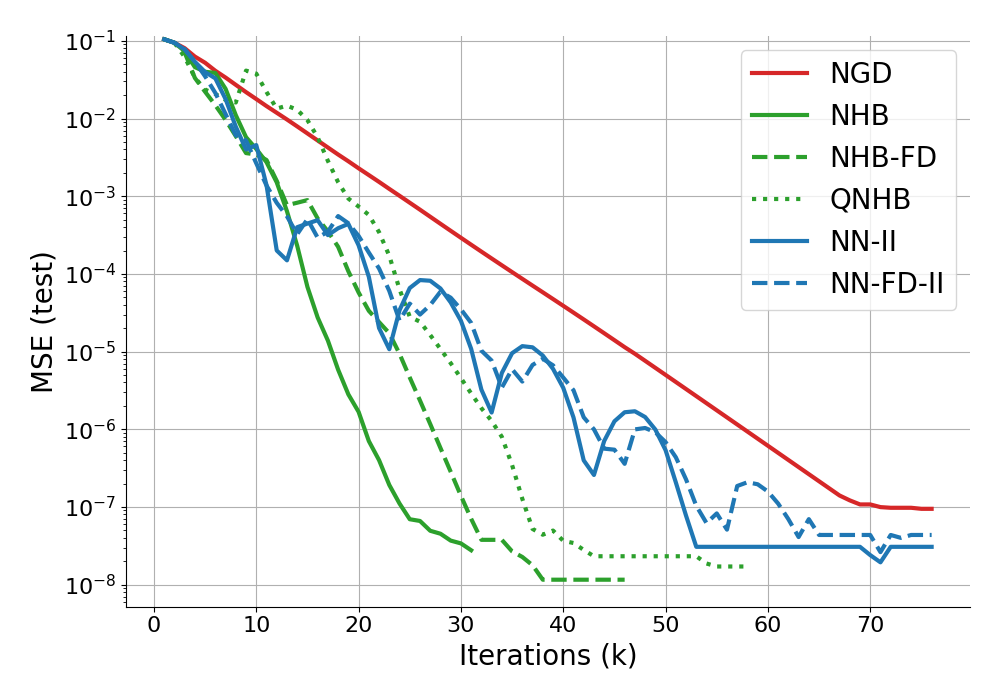}
    \includegraphics[width=0.45\linewidth]{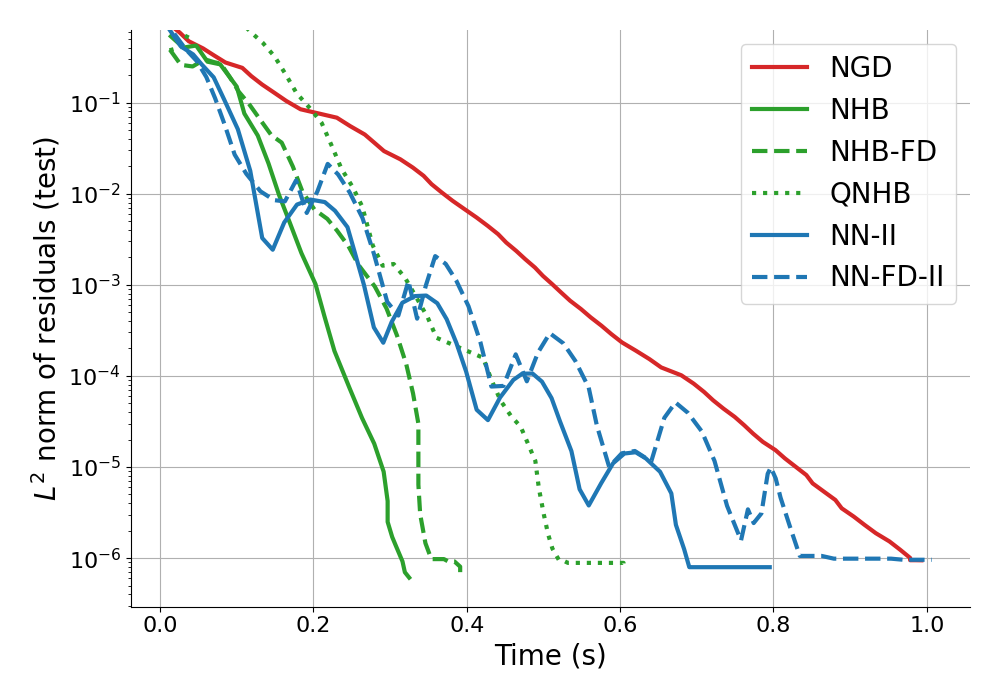}
    \caption{Convergence trends for the linear advection diffusion PDE. The MSE of predictions compared with the ground truth as a function of the number of iterations (left); the $L^2$ norm of residuals with respect to computational time (right). In (red) the baseline method NGD. In (green) the NHB (solid line), NHB-FD (dashed line) and QNHB (dotted line). In (blue) the NN-II (solid line), NN-II-FD (dashed line).}
    \label{fig:pinn_reacdiff_convergence}
\end{figure}

\subsubsection{A non linear reaction diffusion equation}
This is also a 1d example with equation
$$A(\u)(x)\coloneqq-\frac{\partial^2 \u}{\partial x^2}(x) + \u^3(x) = \pi^2\cos(\pi x)-\cos^3(\pi x) =: f(x)$$
in the interval $\Omega=[-1,1]$ with Newman boundary conditions
$$\frac{\partial \u}{\partial x} \Big\rvert_{x=\{-1,1\}}=0.$$
Then the exact solution  is $\u(x)=\cos(\pi x).$

We also use a two hidden layer neural network $\cN\cN_\gamma$ with $5$ neurons in each layer and hyperbolic tangent activation function, where $\gamma$ are the parameters of the neural network. To enforce the boundary conditions we add two more parameters $c, a\in\R$ and transform the output of the network as follows
$$D(\params)(x) = \cN\cN_\gamma(x)(1-x)^2(1+x)^2+c+a\left(\frac{x^3}{3}-x\right),$$
with $\params=(\gamma, a, c)$.
For training we use $1000$ equi-spaced collocation points in the domain $\Omega=[-1,1]$ and $10000$ for testing.
In this example we also see that adding momentum helped to accelerate the convergence. However, the Nesterov versions diverged or behaved similar to NGD unless we modified the momentum constant by setting $\beta_k \to \tilde \beta_k = 0.75\beta_k$ (see \cref{fig:pinn_nonlinear_convergence}).

\begin{figure}[ht]
    \centering
    \includegraphics[width=0.45\linewidth]{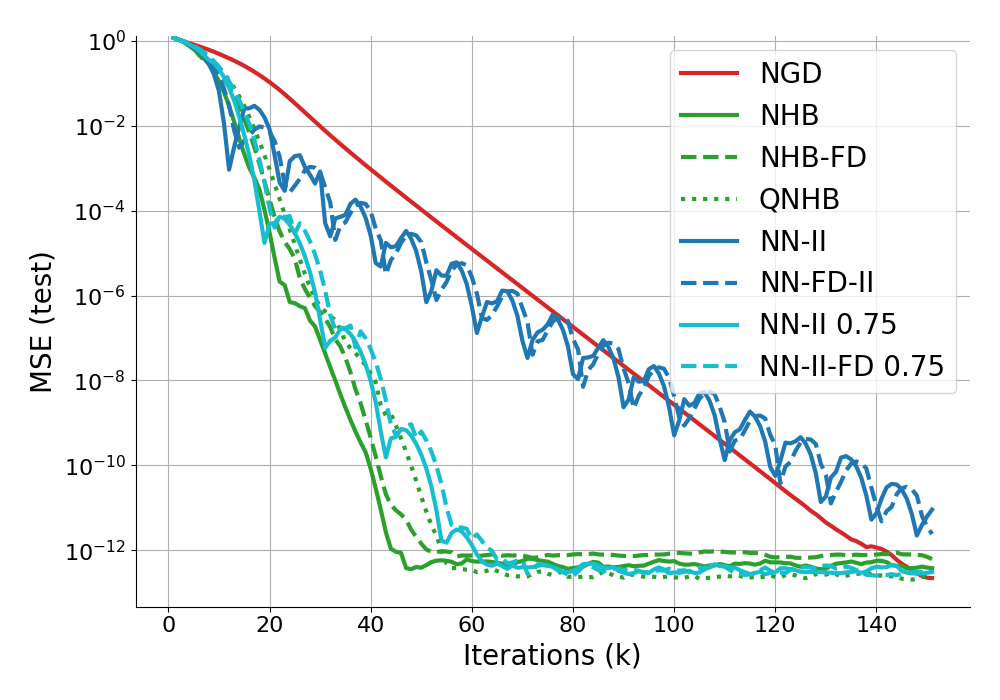}
    \includegraphics[width=0.45\linewidth]{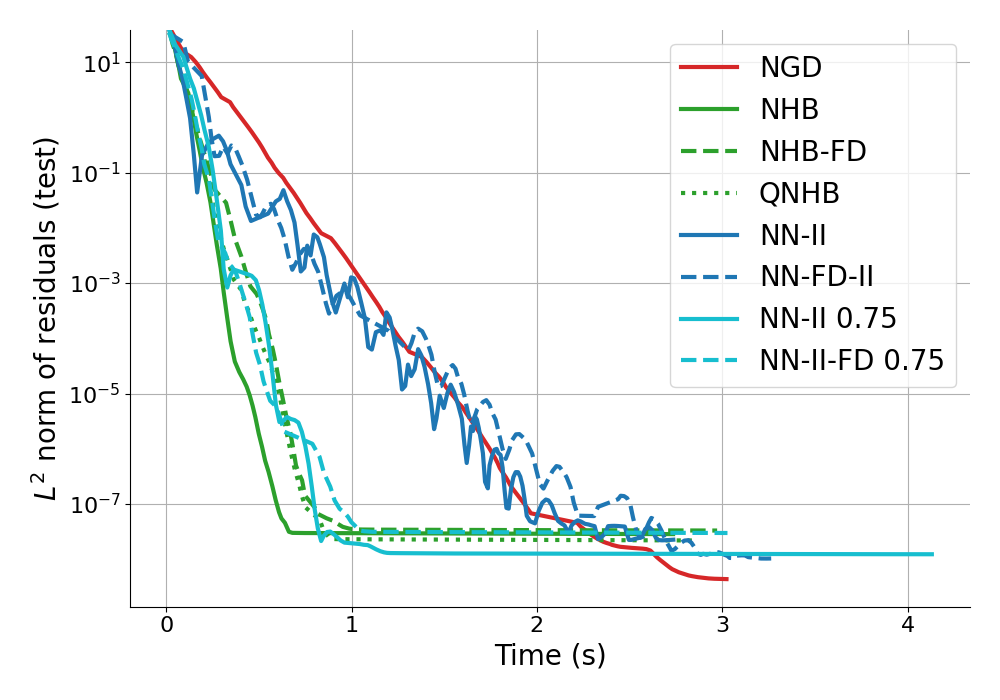}
    \caption{Convergence trends for the non linear advection diffusion PDE. The MSE of predictions compared with the ground truth as a function of the number of iterations (left); the $L^2$ norm of residuals with respect to computational time (right). In (red) the baseline method NGD. In (green) the NHB (solid line), NHB-FD (dashed line) and QNHB (dotted line). In (blue) the NN-II (solid line), NN-II-FD (dashed line). In (cyan) NN-II 0.5 (solid line) and NN-II-FD 0.5 (dashed line) both using the modified Nesterov momentum constant $\tilde \beta_k = 0.75\beta_k$.}
    \label{fig:pinn_nonlinear_convergence}
\end{figure}

\section{Conclusion}

In this work we have seen how to derive in a principled way functional versions of classical momentum optimization algorithms like Heavy-Ball and Nesterov. We have also proposed approximations and alternative formulations some of which are linked to existing natural versions of momentum. Furthermore, we have tested in a variety of settings how these methods can effectively accelerate the convergence of natural gradient descent both in term of number of iterations and computational time.

There are, however, many questions that are left for future analysis. One open question is related to finding a more balanced and theoretically based choice for the learning rate $\alpha_k$ and momentum constant $\beta_k$, for example by leveraging curvature information. Another question is how do these algorithms behave or potentially degrade when we pass to a stochastic setting with reduced batch sizes. Further testing these methods in more complex problems or generalizing these ideas to vector value functions (e.g. multiple classes classification problems) can be another direction to follow.


\FloatBarrier

\section{References}
\printbibliography[heading=none]


\FloatBarrier
\appendix

\section{Proof of \Cref{prop:qnhb} (
Quasi-Natural Heavy-Ball approximation)}
\label{sec:proof-qnhb}
First we develop
\begin{align}
\|\psikT\overline{\pk}-\psikT\pk\|_X &= \beta_k\|\psikT\pkm-\psikT\GXinvk\GXkkm\pkm\|_X \notag\\
&= \beta_k\|\psikT(I-\GXinvk\GXkkm) \pkm\|_X \label{eq:IGG}
\end{align}
where $h_k \Ginvk\nabla L(\paramk)$ cancels out as both methods have this same natural gradient update term. Recalling that $\GXkkm=({\psik, \psikmT})_X$ and using a Taylor expansion, we obtain
$$\psikm = \psi(\paramk - h_{k-1}\pkm) = \psik - h_{k-1}\Hk  + o(h_{k-1}\|\pkm\|_2),$$
with $\Hk\coloneqq H_D(\paramk)(\cdot, \pkm)\in V^d$ such that $\Hk^T \pkm = H_D(\paramk)(\pkm, \pkm)$, we obtain
\begin{align*}
\GXkkm
&=({\psik,\psikmT})_X\\
&=({\psik, \psikT - h_{k-1}\Hk^T)_X + o(h_{k-1}\|\pkm\|_2})\\
&=\GXk - h_{k-1}({\psik, \Hk^T})_X + o(h_{k-1}\|\pkm\|_2).
\end{align*}
Now we can write
\begin{align*}
    \psikT\left(I-\GXinvk\GXkkm\right)\pkm
    &=  h_{k-1}\psikT\GXinvk({\psik,\Hk^T \pkm})_X + o(h_{k-1}\|\pkm\|^2_2) \\
    &= h_{k-1} P_{\cT_k}^X H_D(\paramk)(\pkm,\pkm)  + o(h_{k-1}\|\pkm\|^2_2).
\end{align*}
Therefore
\begin{align*}
    \|\psikT(I-\GXinvk\GXkkm) \pkm\|_X
   &=  h_{k-1} \Vert P_{\cT_k}^X \Hk^T \pkm \Vert_X +  o(h_{k-1}\|\pkm\|^2_2) \\
   &\le h_{k-1} \Kmax \|\pkm\|^2_2 + o(h_{k-1} \|\pkm\|^2_2),
\end{align*}
where we have used the fact that $P_{\cT_k}^X$ is an orthogonal projection and  $\| \Hk^T \pkm \|_X = \|H_D(\paramk)(\pkm, \pkm)\|_X \leq \Kmax \|\pkm\|_2$.

\section{Proof of \Cref{prop:qnhb-bis} (Natural Heavy-Ball with Functional Difference approximation)}
\label{sec:proof-nhb-fd}
We have
\begin{align*}
    \vk &= D(\paramkm + h_{k-1}\pkm  ) \\
    &= \vkm + h_{k-1}\psikmT\pkm + \frac{h_{k-1}^2}{2}H_D(\paramkm)(\pkm,\pkm) +  o(h_{k-1}^2\|\pkm\|_2^2).
\end{align*}
Then, letting $\Delta_k  \coloneqq H_D(\paramkm)(\pkm,\pkm) \in V$, we have
\begin{align*}
    z^{(k-1)}&=({\psik, \vk-\vkm})_X\\
    &=({\psik, h_{k-1}\psikmT\pkm + \frac{h_{k-1}^2}{2}\Delta_k {  + o(h_{k-1}^2\|\pkm\|_2^2)}})_X\\
    &=h_{k-1}({\psik, \psikmT})_X\pkm + \frac{h_{k-1}^2}{2}({\psik, \Delta_k})_X {  +    o(h_{k-1}^2\|\pkm\|_2^2)}\\
    &=h_{k-1}\GXkkm\pkm + \frac{h_{k-1}^2}{2}({\psik, \Delta_k})_X {  +   o(h_{k-1}^1\|\pkm\|_2^2)}.
\end{align*}
Inserting this in the update difference, we obtain
\begin{align*}
\left\|\psikT\widehat{\pk}-\psikT\pk\right\|_X
&= \beta_k\left\|\psikT\GXinvk \frac{z^{(k-1)}}{h_{k-1}}-\psikT\GXinvk\GXkkm\pkm\right\|_X\\
&= \beta_k\left\|\psikT\GXinvk\left(\frac{z^{(k-1)}}{h_{k-1}}-\GXkkm\pkm\right)\right\|_X\\
&= \beta_k\left\|\frac{h_{k-1}}{2}\psikT\GXinvk ({\psik,\Delta_k})_X+o(h_{k-1}\|\pkm\|_2^2)  \right\|_X \\
&= \beta_k \frac{h_{k-1}}{2} \| P^X_{\cT_k}\Delta_k \|_X +   o(h_{k-1}\|\pkm\|_2^2)    \\
&\le \beta_k \frac{h_{k-1}}{2} \kappa_{max}(\paramkm) \| \pkm\|_2^2 + o(h_{k-1}\|\pkm\|_2^2)
\end{align*}
where we have used the fact that $P^X_{\cT_k}$ is an orthogonal projection and $\|\Delta_{k}\|_X = \|H_D(\paramkm)(\pkm, \pkm)\|_X \leq \kappa_{max}(\paramkm) \| \pkm\|_2^2$.

\section{Comparison between natural Nesterov alternative formulations}
\label{sec:appendix-nesterov}

We did not include the first proposed version of natural Nesterov \Cref{def:nn-I} NN-I to resent more readable plots and because its convergence time was never better than NGD as we can see in the following figures: \Cref{fig:makey-nesterov,fig:or-nesterov,fig:gn-or-nesterov}. This time we also include the natural Nesterov versions stemming from \Cref{rmk:quasi-nesterov} that we name QNN-I and QNN-II. Except in the classification problem when using $H_\cL$ as metric for the gradient term where NN-II blows-up, NN-II behaves in terms of iterations to convergence similarly as NN-I but taking around the double in computational time. The alternative variants from \Cref{rmk:quasi-nesterov} blow-up or take even more iteration to convergence than the function difference variants. All of this suggests the use of NN-II as the preferred method.

\begin{figure}[ht]
    \centering
    \includegraphics[width=0.45\linewidth]{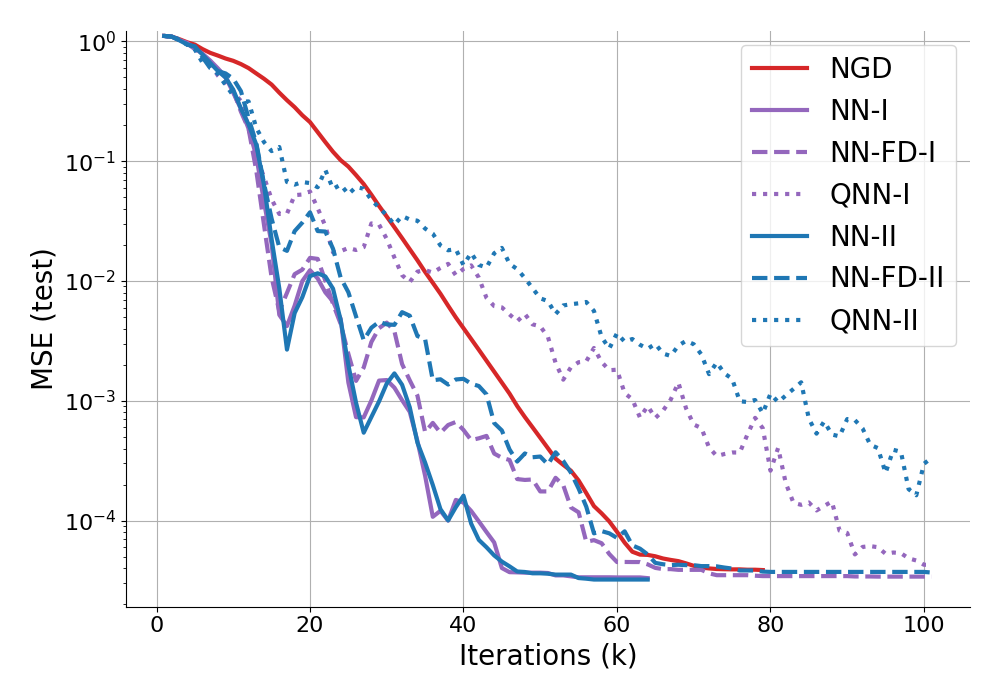}
    \includegraphics[width=0.45\linewidth]{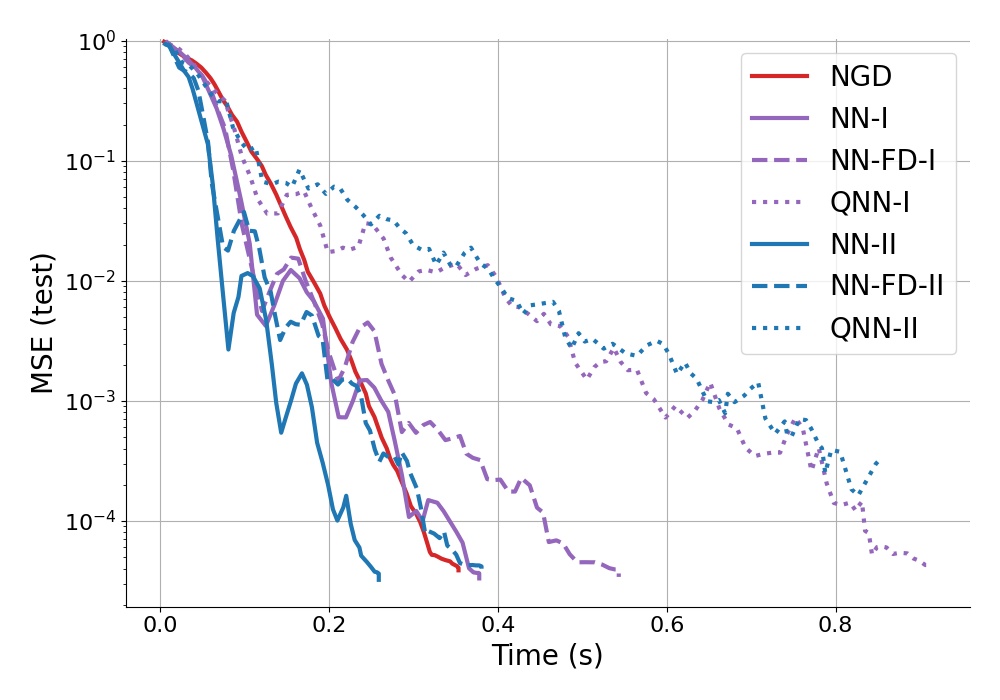}
    \caption{Convergence trends for the Mackey-Glass problem as a function of the number of iterations (left) or the time in seconds (right). In (red) the baseline method NGD. In (purple) the NN-I (solid line), NN-I-FD (dashed line), QNN-I (dotted-line). In (blue) the NN-II (solid line), NN-II-FD (dashed line), QNN-II (dotted-line).
    The curves represent the average of $50$ repetitions at random initializations. We see that all strategies behave similarly with NN-I alternative clearly more computationaly expensive.}
    \label{fig:makey-nesterov}
\end{figure}

\begin{figure}[ht]
    \centering
    \includegraphics[width=0.45\linewidth]{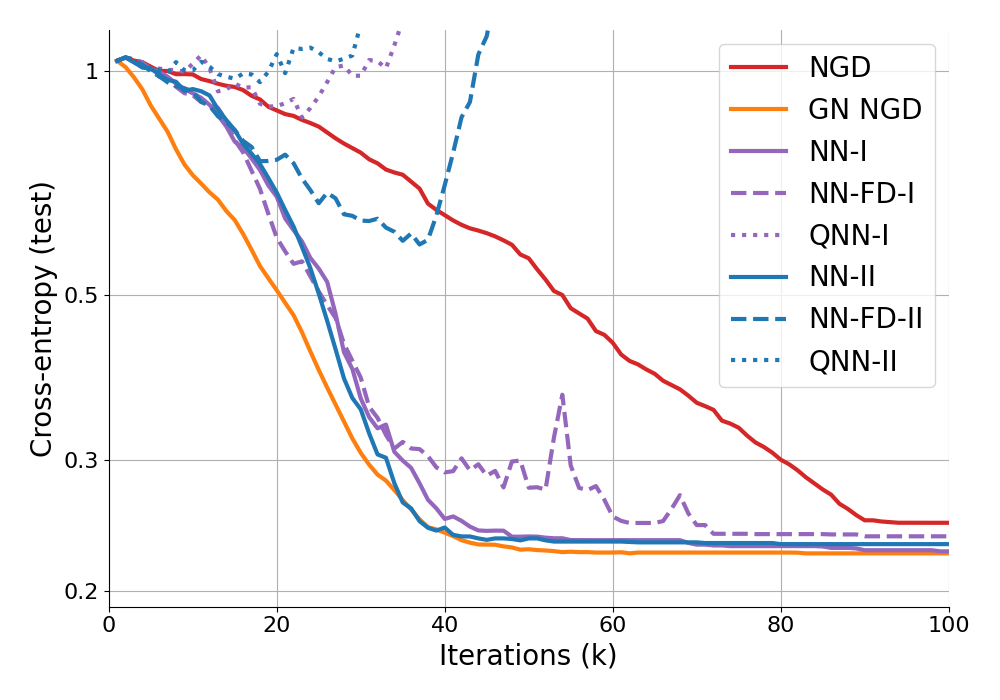}
    \includegraphics[width=0.45\linewidth]{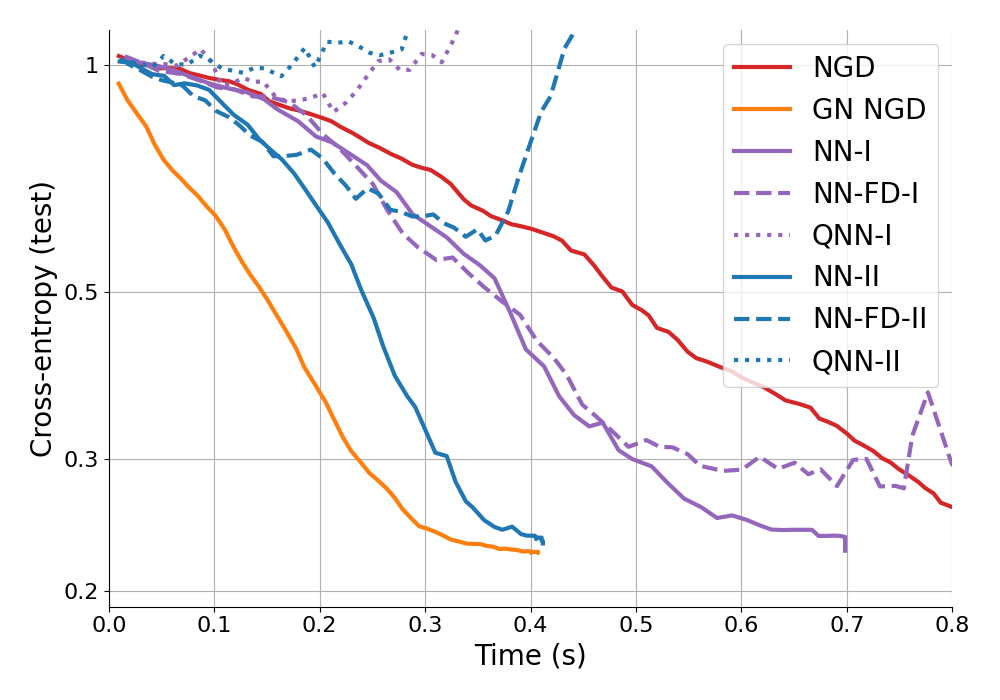}
    \caption{Convergence trends for the classification problem as a function of the number of iterations (left) or the time in seconds (right). The baseline methods are NGD (red) and Gauss-Newton NGD (GN-NGD) (orange). In (purple) the NN-I (solid line), NN-I-FD (dashed line), QNN-I (dotted-line). In (blue) the NN-II (solid line), NN-II-FD (dashed line), QNN-II (dotted-line).
    The curves represent the average of $50$ repetitions at random initializations. We see that all strategies behave similarly with NN-I alternative clearly more computationaly expensive.}
    \label{fig:or-nesterov}
\end{figure}

\begin{figure}[ht]
    \centering
    \includegraphics[width=0.45\linewidth]{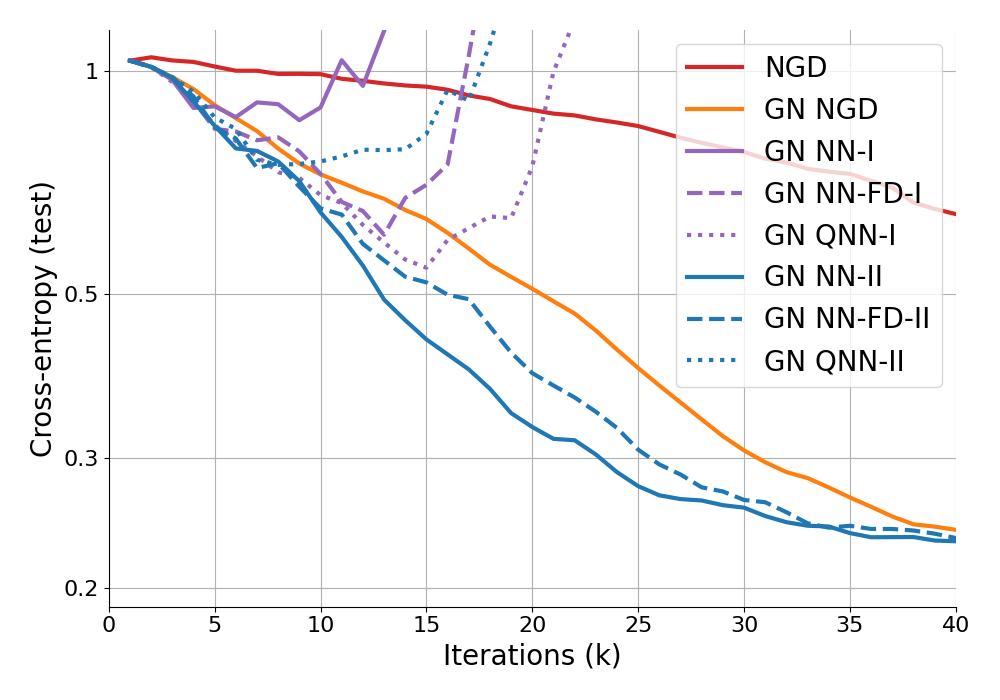}
    \includegraphics[width=0.45\linewidth]{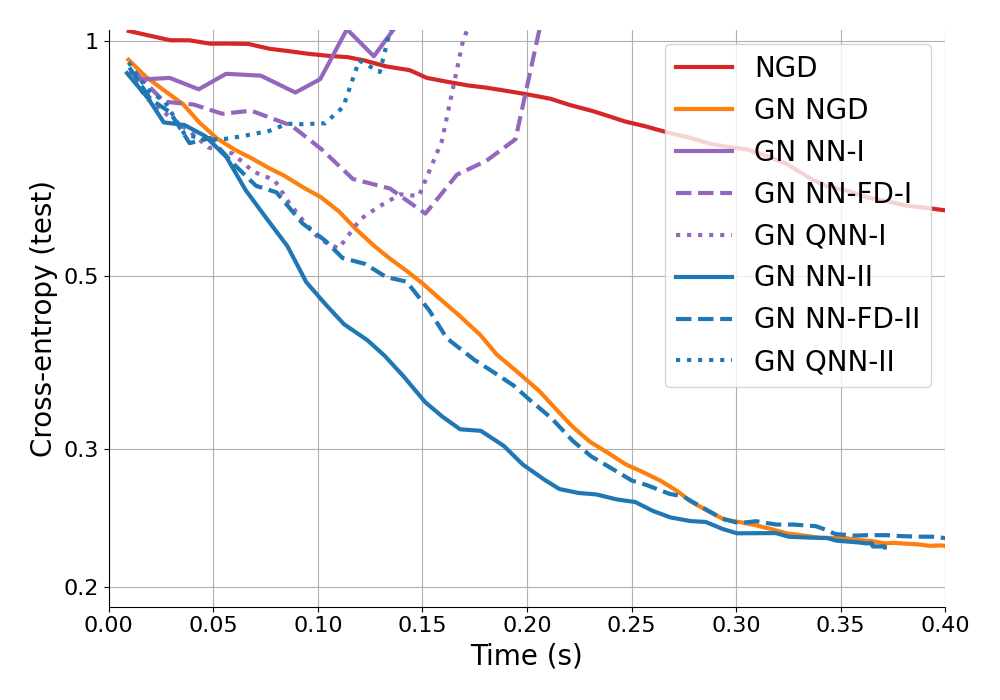}
    \caption{Convergence trends for the classification problem as a function of the number of iterations (left) or the time in seconds (right). The baseline methods are NGD (red) and Gauss-Newton NGD (GN-NGD) (orange). In (purple) the GN NN-I (solid line), GN NN-I-FD (dashed line), GN QNN-I (dotted-line). In (blue) the GN NN-II (solid line), GN NN-II-FD (dashed line), GN QNN-II (dotted-line). The curves represent the average of $50$ repetitions at random initializations. In this case the Gauss-Newton versions of NN-II all present divergent trends.}
    \label{fig:gn-or-nesterov}
\end{figure}


\end{document}